\documentclass[11pt]{article}
\usepackage{amsfonts,amsthm,mathrsfs,amssymb}
\usepackage{amsmath}
\usepackage[page,title,titletoc,header]{appendix}
\usepackage{color}
\usepackage{authblk}
\usepackage{graphicx}
\usepackage{caption,subcaption}

\usepackage[a4paper,hmargin=1.25in,vmargin=1in]{geometry}

\numberwithin{equation}{section}
\newtheorem{thm}{Theorem}[section]
\newtheorem{lemma}[thm]{Lemma}
\newtheorem{pro}[thm]{Proposition}
\newtheorem{corollary}[thm]{Corollary}

\newtheorem{rem}[thm]{Remark}

\newtheorem{as}{Assumption}
\newtheorem{alg}{Algorithm}

\newcommand{\be}{\begin{equation}}
\newcommand{\ee}{\end{equation}}
\newcommand{\bea}{\begin{eqnarray*}}
\newcommand{\eea}{\end{eqnarray*}}
\newcommand{\mR}{\mathbb{R}}
\newcommand{\mN}{\mathbb{N}}
\newcommand{\mE}{\mathbb{E}}

\newcommand{\mcE}{\mathcal{E}}

\newcommand{\mcN}{\mathcal{N}}
\newcommand{\mcT}{\mathcal{T}}
\newcommand{\mcS}{\mathcal{S}}

\newcommand{\mcX}{\mathcal{X}}
\newcommand{\mcV}{\mathcal{V}}
\newcommand{\mcH}{\mathcal{H}}
\newcommand{\mcK}{\mathcal{K}}

\newcommand{\mbK}{\mathbf{K}}
\newcommand{\mbR}{\mathbf{R}}

\newcommand{\mbb}{\mathbf{b}}
\newcommand{\mbx}{\mathbf{x}}
\newcommand{\mby}{\mathbf{y}}

\newcommand{\TK}{\mathcal{T}}
\newcommand{\TX}{\mathcal{T}_{\bf x}}
\newcommand{\TXt}{\mathcal{T}_{\tilde{\bf x}}}
\newcommand{\SX}{\mathcal{S}_{\bf x}}
\newcommand{\TKL}{\mathcal{T}_{\lambda}}
\newcommand{\TXL}{\mathcal{T}_{{\bf x}\lambda}}
\newcommand{\LK}{\mathcal{L}}
\newcommand{\IK}{\mathcal{I}_{\rho}}

\newcommand{\Ptx}{P_{\mbtx}}

\newcommand{\mbtx}{\mathbf{\tilde{x}}}
\newcommand{\TTX}{\mathcal{T}_{\mbtx}}
\newcommand{\STX}{\mathcal{S}_{\mbtx}}

\newcommand{\HK}{H_K}
\newcommand{\LR}{L_{\rho}^2}
\newcommand{\HKx}{H_{\mbtx}}

\newcommand{\J}{{\bf J}}

\newcommand{\FH}{f_{\mathcal{H}}}

\DeclareMathOperator{\spn}{span}
\DeclareMathOperator{\rank}{rank}
\DeclareMathOperator*{\argmin}{arg\,min}

\DeclareMathOperator*{\sign}{sign}
\DeclareMathOperator*{\diag}{diag}
\newcommand{\tr}{\operatorname{tr}}
\newcommand{\la}{\langle}
\newcommand{\ra}{\rangle}

\newcommand{\eref}[1] {(\ref{#1})}

\begin{document}
\title{Optimal Rates for Learning with Nystr\"{o}m Stochastic Gradient Methods\thanks{This work was done when J.L. was working in the LCSL, IIT@MIT. J.L. is now with the LIONS, EPFL. (jhlin5@hotmail.com)}}

\renewcommand\Affilfont{\tiny}
\author{Junhong Lin$^\dag\qquad$ Lorenzo Rosasco$^{\dag\ddag}$\\
{\small{\em junhong.lin@iit.it}$~~\qquad$ {\em lrosasco@mit.edu}}$~\quad$\\
{\scriptsize $^\dag$LCSL, Massachusetts Institute of Technology and Istituto Italiano di Tecnologia, Cambridge, MA 02139, USA}\\
{\scriptsize $^\ddag$DIBRIS, Universit\`a degli Studi di Genova, Via Dodecaneso 35, Genova, Italy}\\
}
\maketitle \baselineskip 16pt

\begin{abstract}
 In the setting of nonparametric regression, we propose and study a combination of  stochastic gradient methods with Nystr\"{o}m subsampling, allowing multiple passes over the data and mini-batches.
 Generalization error bounds for the studied algorithm are provided. Particularly, optimal learning rates are  derived considering different possible choices of the step-size, the mini-batch size, the number of iterations/passes, and the subsampling level. In comparison with state-of-the-art algorithms such as the classic stochastic gradient methods
 and kernel ridge regression with Nystr\"{o}m, the studied algorithm has  advantages on the computational complexity, while achieving the same optimal learning rates. Moreover, our results indicate that using mini-batches can reduce the total computational cost, while achieving the same optimal statistical results. 
\end{abstract}


\section{Introduction}
In supervised learning, given  a sample of $n$ pairs of inputs and outputs, the goal is to estimate a function to be used to predict future outputs based on observing only the corresponding inputs. The quality of an estimate is often  measured in terms of the mean-squared prediction error, in which case the regression function is optimal.

Since the properties of the function to be estimated are not known a priori, nonparametric techniques,  that can adapt their complexity to the problem at hand,  are often key to good results.
Kernel methods \cite{friedman2001elements,shawe2004kernel} are probably the most common nonparametric approaches to learning. They are based on choosing a reproducing kernel Hilbert space (RKHS) as the hypothesis space in the design of learning algorithms.
A classical learning algorithm using kernel methods to perform learning tasks is  kernel ridge regression (KRR), which is based on minimizing
the sum of a data-fitting term and an explicit penalty term. The  penalty term is used for regularization, and controls the complexity of the solution, preventing overfitting.
The statistical properties of KRR have  been studied extensively, see e.g. \cite{caponnetto2007optimal,steinwart2008support}, and are known to be optimal in a minmax sense \cite{Tsybakov2008introduction}. The drawbacks of KRR are mainly computational. 
Indeed, a standard implementation of KRR requires the computation of a linear system defined by a kernel matrix, which thus requires costs $O(n^3)$ in time and $O(n^2)$ in memory, where $n$ is the number of points. Such scalings are prohibitive when in large scale scenario, where the sample size $n$ is large. A possible alternative is considering learning algorithms based on  iterative procedure \cite{engl1996regularization,zhang2005boosting,yao2007early}. In this  kind of learning algorithms, an empirical objective function is optimized in an iterative way with no explicit constraint or penalization, and the regularization against overfitting is realized by early-stopping the empirical procedure. Early-stopping has certain computational advantage over KRR, as it does not require the computation of the inverse of a kernel matrix.
Indeed, if the algorithm stops after $T$ iterations, the aggregate time complexity is $O(Tn^2)$ for gradient descent \cite{yao2007early,raskutti2014early} and conjugate gradient methods \cite{blanchard2010optimal}, while $O(Tn)$ for stochastic gradient methods (SGM) \cite{rosasco2015learning,lin2016optimal}.

Although the statistical aspects of early-stopping procedures are well understood, either the computation or the storage of these algorithms can be challenging for large datasets.
Indeed, the storage and/or computational cost of these algorithms, are/is at least quadratic in the number of training examples, due to the storage and/or calculation of a fully empirical kernel matrix.
To avoid storing and/or computing a large kernel matrix, a natural approach is to replace the standard kernel matrix with a smaller matrix obtained by subsampling \cite{smola2000sparse,williams2000using}. Such an approach, referred to as Nystr\"{o}m method in machine learning,
  provides one of the main approaches towards kernel methods with large scale learning. Particularly, Nystr\"{o}m techniques are successfully used together with KRR \cite{rudi2015less,yang2015randomized} while achieving optimal statistical results \cite{rudi2015less} in the random design setting. Moreover, it has recently been combined with early-stopping on batch gradient methods,
   and optimal statistical results in the fixed design setting are provided \cite{camoriano2016nytro}.

In this paper, we investigate stochastic gradient methods with  Nystr\"{o}m subsampling (named as NySGM) in the nonparametric regression setting.
At each iteration, NySGM updates its current solution by subtracting a scaled gradient estimate over a mini-batch of points drawn uniformly at random from the sample, and subsequently projecting onto an ``empirically subsampling" space.
The subsampling level, the number of iterations/passes, the step-size and the mini-batch size are then the free parameters to be determined. Our main results show how can these parameters be chosen so that the corresponding solutions achieve optimal learning errors in a variety of settings.
In comparisons with state-of-the-art algorithms such as Nystr\"{o}m KRR and classic SGM, NySGM has the advantage either on the computation or on the storage, while achieving the same optimal error bounds, see Section \ref{sec:dis} for details.
For example, in the special case  that no benign assumptions on the problem \cite{zhang2005learning} are made,
NySGM with suitable choices of parameters can lead to optimal learning rates $O(n^{-0.5})$ after one pass over the data, where the costs are $O(n^{1.5})$ in time and $O(n^{1.5})$ in memory,
compared to $O(n^{2})$ in time and $O(n^{1.5})$ in memory for Nystr\"{o}m KRR.  Moreover, as will be seen in Section \ref{sec:main}, our results indicate that using mini-batches can reduce the total computational cost, while achieving the same optimal statistical results. Such a result is somewhat surprising, as it is well-known that using mini-batches does not reduce the total computational cost for classical SGM.
The proof for our main results is based on tools from concentration inequalities, operator theory and convex analysis,
and it borrows idea from, e.g., \cite{yao2007early,smale2007learning,bauer2007regularization,ying2008online,rudi2015less}.

The rest of this paper is organized as follows. In the next section, we introduce the nonparametric regression setting and NySGM.
In Section \ref{sec:main}, we present our theoretical results following with simple discussions. In Section \ref{sec:dis}, we discuss and compare our results with related work.
All proofs for related results and equalities of this paper are given in Section \ref{sec:proofs} and the appendix.

 \section{Learning with Nystr\"{o}m Stochastic Gradient Methods}\label{sec:learning}
In this section, we first describe the learning setting and then introduce the studied algorithm.
\subsection{Learning Problems}
We consider a supervised learning problem.
Let $\rho$ be a probability measure on a measure space $Z=X\times \mR,$ where  $X$ is the input space and $\mR$ is the output space. Here, $\rho$ is fixed but unknown. Its information can be only known through a sample
$\mathbf z=\{z_i=(x_i, y_i)\}_{i=1}^n$ of $n\in\mN$ points, which we assume to be i.i.d..

 Kernel methods are based on choosing a hypothesis space as a reproducing kernel Hilbert space (RKHS) associated with a reproducing kernel. Recall that a reproducing kernel $K$ is a symmetric function $K: X
\times X \to \mR$ such that $(K(u_i, u_j))_{i, j=1}^\ell$ is
positive semidefinite for any finite set of points
$\{u_i\}_{i=1}^\ell$ in $X$. The reproducing kernel $K$ defines a RKHS $(\HK, \|\cdot\|_{\HK})$ as the
completion of the linear span of the set $\{K_x(\cdot):=K(x,\cdot):
x\in X\}$ with respect to the inner product $\la K_x,
K_u\ra_{K}:=K(x,u).$

Given only the sample $\bf z$, the goal is to solve the following expected risk minimization problem,
\be\label{erm}
\inf_{f \in \HK} \mcE(f), \quad \mcE(f) = \int_{Z} (f(x) - y)^2 d\rho(z).
\ee

\subsection{Nystr\"{o}m Stochastic Gradient Method}
To solve the expected risk minimization problem,  in this paper, we propose the following SGM, using mini-batches and Nystr\"{o}m subsampling.
For $t \in \mN,$ the set $\{1, ..., t\}$ of the first $t$ positive integers is denoted by $[t]$.
\begin{alg}\label{alg:1}
Let $b \in \mN.$
Given any $\bf z$, let $\mbtx = \{{x}_1, {x}_2,\cdots,{x}_m\}$ with $m \leq n$. Let $\Ptx$ be the projection operator with its range as the subspace $\HKx=\spn\{K_{{x}_i}: i \in [m]\}.$
The Nystr\"{o}m stochastic gradient method (abbreviated as NySGM) is defined by $f_1 =0$ and
\be\label{eq:alg1}
f_{t+1}=f_t - \eta_t {1 \over b} \sum_{i= b(t-1)+1}^{bt} (f_t(x_{j_i}) - y_{j_i}) \Ptx (K_{x_{j_i}}) , \qquad t=1, \ldots, T, \ee
where $\{\eta_{t}>0\}_{t\in \mN}$ is a step-size sequence.  Here, $j_1,j_2,\cdots,j_{bT}$ are i.i.d. random variables from the uniform distribution on $[n]$. \footnote{The random variables $j_1,\cdots, j_{bT}$ are conditionally independent given the sample $\bf z$.}
\end{alg}
At each iteration, the above algorithm updates its current solution by subtracting a scaled gradient estimate and then projecting onto $\HKx$.
In comparison with the classic SGM from \cite{lin2016optimal}, the studied algorithm has an extra projection step in its iterative relationship.  The projection step is a result of the subsampling technique. It ensures that the learning sequence always lies in $\HKx$, a smaller space than $\mbox{span} \{K_{{\bf x}_i}: i \in [n]\}$. When $m=n$, the above algorithm is exactly the classic SGM studied in \cite{lin2016optimal}.

Note that there are not any explicit penalty terms in \eref{eq:alg1}, in which case one does not need to tune the penalty parameter, and the only free parameters are the subsampling level $m$, the step-size $\eta_t$, the mini-batch size $b$ and the total number of iterations $T$. Different choices of these parameters can lead to different strategies. In the coming subsection, we are particularly interested in the fixed step-size setting, i.e., $\eta_t = \eta$ for some $\eta>0$, with $b=1$ or $\sqrt{n}$.

The total number of iterations $T$ can be bigger than the sample size $n$, which means that the algorithm can use the data more than once, or in another words, we can run the algorithm with multiple passes over the data. Here and in what follows, the number of `passes' over the data is referred to $\lceil {bt \over m} \rceil$ at $t$-th iteration of the algorithm.

The aim of this paper is to derive generalization error bounds, i.e., the excess risk $\mcE(f_{T+1}) - \inf_{f\in \HK} \mcE(f),$ for the above algorithm.
Throughout this paper, we assume that $\{\eta_t\}_t$ is non-increasing and $T \in \mN$ with $T \geq 3$.
We denote by $\J_t$ the set $\{j_l: l=b(t-1)+1,\cdots,bt\}$ and by $\J$ the set $\{j_l: l=1,\cdots,bT\}$.

\subsection{Numerical Realizations}
Algorithm 1 has different equivalent forms, which are easier to be implemented for numerical simulations. For any finite subsets $\mbx$ and $\mbx'$ in $X$, denote
 the cardinality of the set $\bf x$ by $| \bf x|,$ and
the $|{\mbx}| \times |\mbx'|$ kernel matrix $[K(x,x')]_{x\in \mbx,x'\in\mbx'}$ by
$\mbK_{\mbx \mbx'}$.
 Let $\mbR \in \mR^{m \times \rank(\mbK_{\mbtx\mbtx})}$ be such that $\mbR\mbR^\top = \mbK_{\mbtx\mbtx}^\dag $.
Then as will be shown in Subsection \ref{equivalent}, Algorithm \ref{alg:1} is equivalent to,
with $\mathbf{b}_1 = \mathbf{0} \in \mR^{\rank(\mbK_{\mbtx\mbtx})},$
\be\label{eq:algNumRea}
\begin{cases}
f_{t} = \sum_{i=1}^m \mbR(i,:) \mathbf{b}_{t} K_{x_i}\\
  \mbb_{t+1} = \mbb_{t} - {\eta_t \over b}\mbR^{\top} \sum_{i= b(t-1)+1}^{bt} (\mbK_{\mbtx x_{j_i}} \mbK_{\mbtx x_{j_i}}^{\top}  \mbR \mbb_{t}   - y_{j_i} \mbK_{\mbtx x_{j_i}})  .
\end{cases}
\ee
Here, $\mbb_t \in \mR^{\rank(\mbK_{\mbtx\mbtx})}$,  and $\mbR(i,:)$ denotes the $i$-th row of the matrix $\mbR$. Assuming that $\rank(\mbK_{\mbtx\mbtx}) \simeq m$ and that the cost of evaluating the kernel on a pair of sample points is $O(d)$, if the computer computes and stores $\mbR$ and $\mbK_{\mbtx \mbx}$ in the preprocessing and then updates $\mbb_t$ by \eref{eq:algNumRea} based on $\mbK_{\mbtx \mbx},\mbR$ and  ${\bf y},$
the space and time complexities for training this algorithm are
 \be\label{eq:cost1}
 O(nm)\quad \mbox{and} \quad O(nmd + m^3 + m^2 T + m b T),
 \ee
 respectively.
 Alternatively, if the computer computes and stores $\mbR$  in the preprocessing and then updates $\mbb_t$ by \eref{eq:algNumRea} based on $\mbR$ and  ${\bf z},$
 the space and time complexities for training are
 \be\label{eq:cost2}
 O(m^2 + nd)\quad \mbox{and}\quad O(m^3 + m^2 d + m^2 T + mdbT),
 \ee
respectively.

To see the performance of Algorithm \ref{alg:1}, we carried out some simple numerical simulations on a simple problem.
We constructed $n=100$ i.i.d. training examples of the form $y= f_{\rho}(x_i) + \omega_i$.
Here, the regression function is $f_{\rho}(x) = |x - 1/2| - 1/2,$
the input $x_i$ is uniformly distributed in $[0,1],$ and $\omega_i$ is a Gaussian noise with zero mean and standard deviation $1,$ for each $i\in [n].$
For all the simulations, the RKHS is associated with a Gaussian kernel $K(x,x') = \exp(-(x-x')^2/(2\sigma^2))$ where $\sigma=0.2$,
the mini-batch size $b=1$, and the step-size $\eta_t = {1/(8n)}$, as suggested by Corollary \ref{cor:1} in Section \ref{sec:main}.
For each subsampling level $m \in \{2,4,6,8,10,12\}$, we ran NySGM 50 times.
The mean and the standard deviation of the computed generalization errors over $50$ trials with respect to the number of passes are depicted in Figure \ref{fig:1}.
Here, the (approximated) generalization errors were computed over an empirical measure with $2000$ points.
As we see from the plots, NySGM performs well when the subsampling level $m \geq 8.$
Moreover, the minimal generalization error is achieved after some number of passes,
and it is comparable with $0.281$ of KRR using cross-validation.
\begin{figure}[h]
    \centering
    \begin{subfigure}[]
    {0.3\textwidth} 	
    \includegraphics[width=\textwidth]{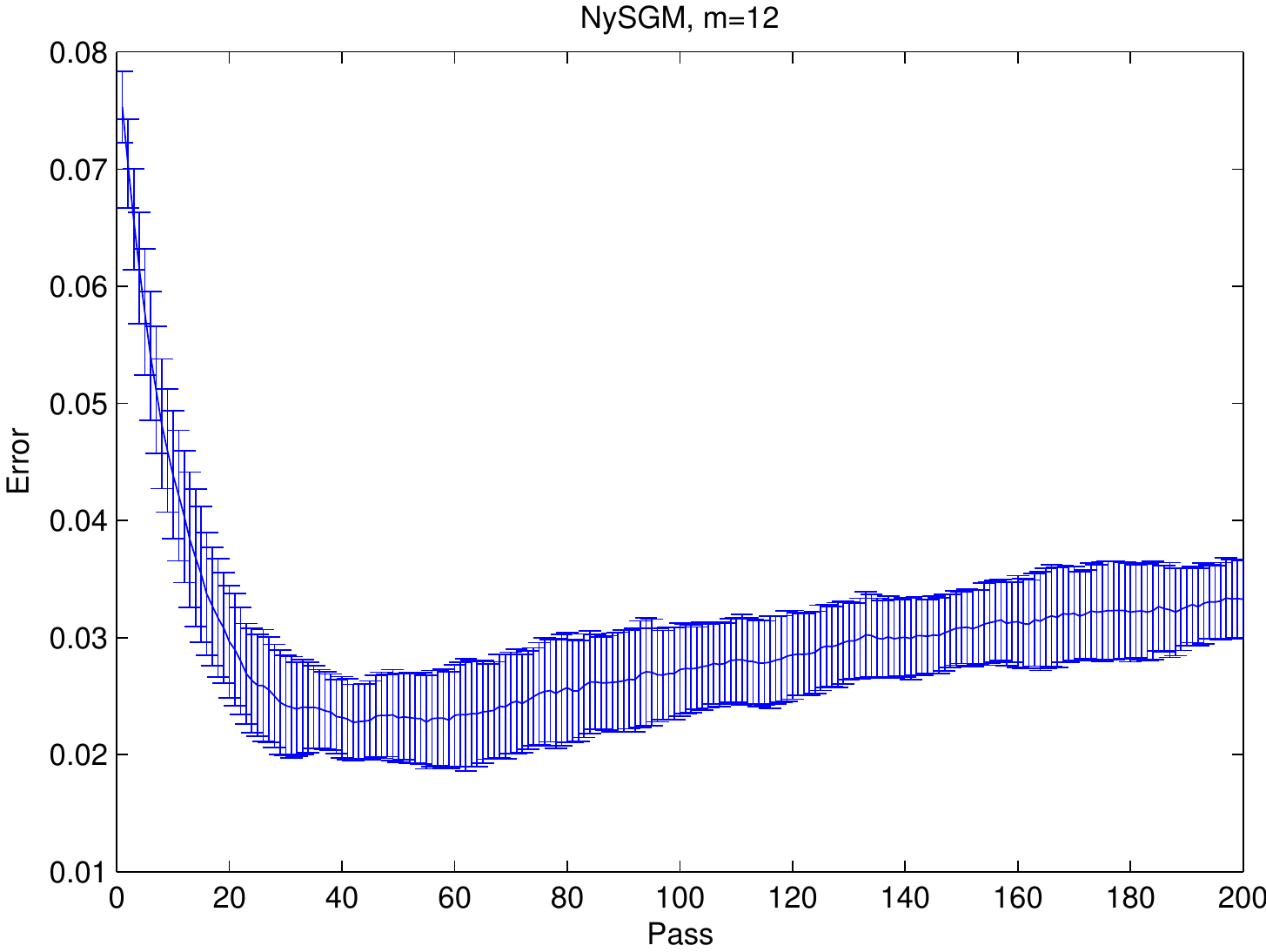}
    \end{subfigure}
     ~
    \begin{subfigure}[]
    {0.3\textwidth}
    \includegraphics[width=\textwidth]{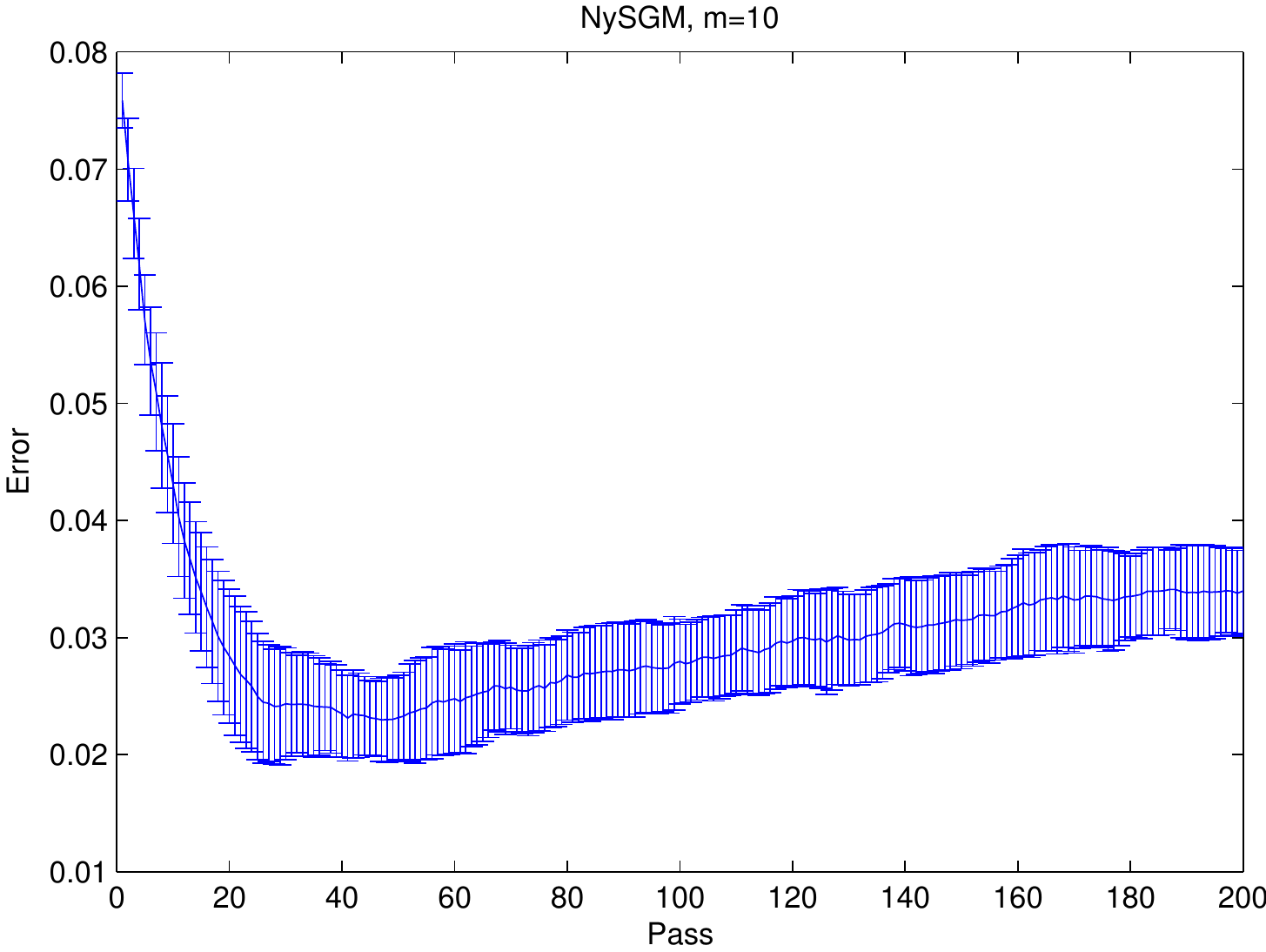}
    \end{subfigure}
    ~
    \begin{subfigure}[]
    {0.3\textwidth}
    \includegraphics[width=\textwidth]{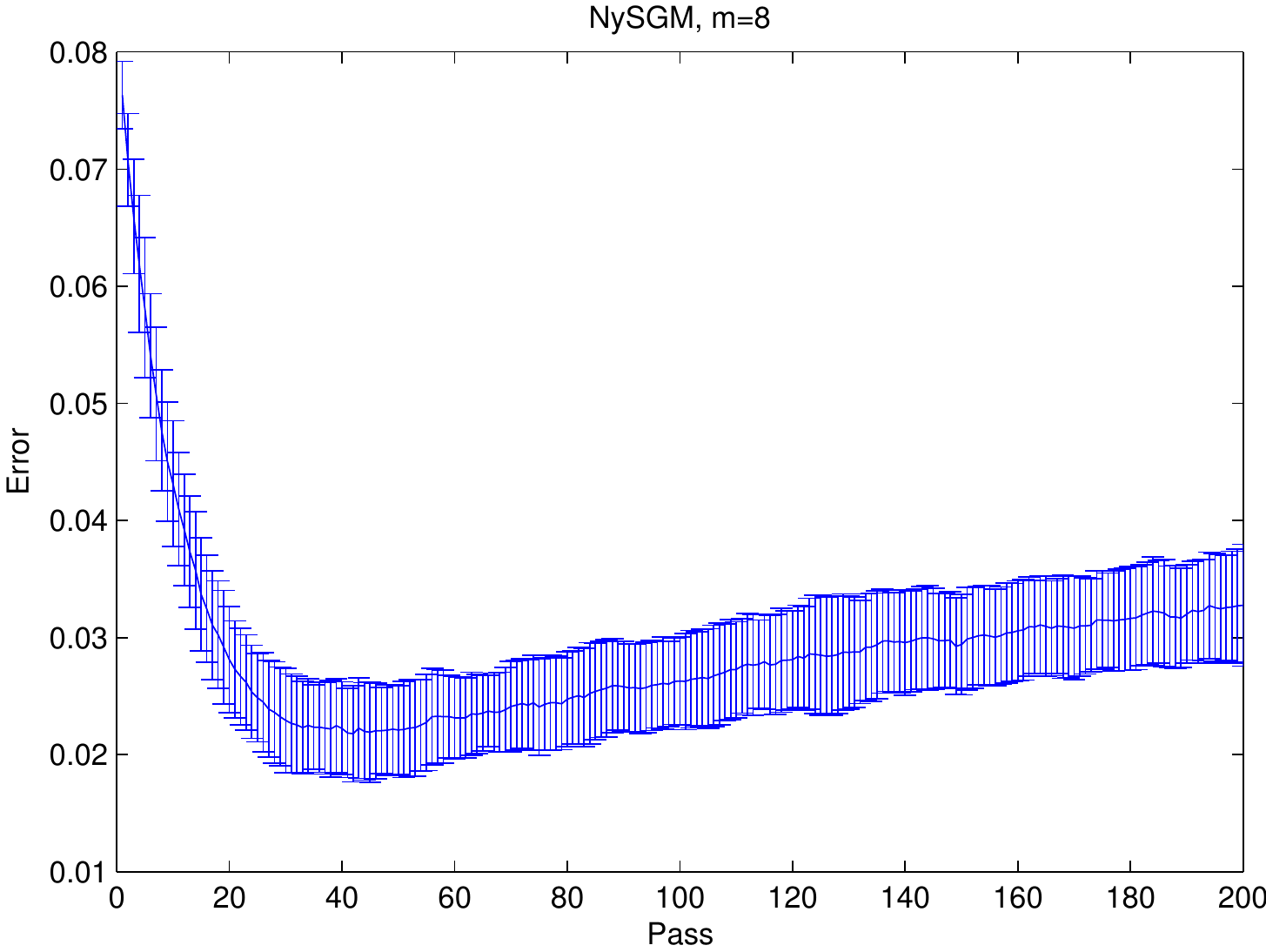}
    \end{subfigure}
    \\
    \begin{subfigure}[]
    {0.3\textwidth} 	
    \includegraphics[width=\textwidth]{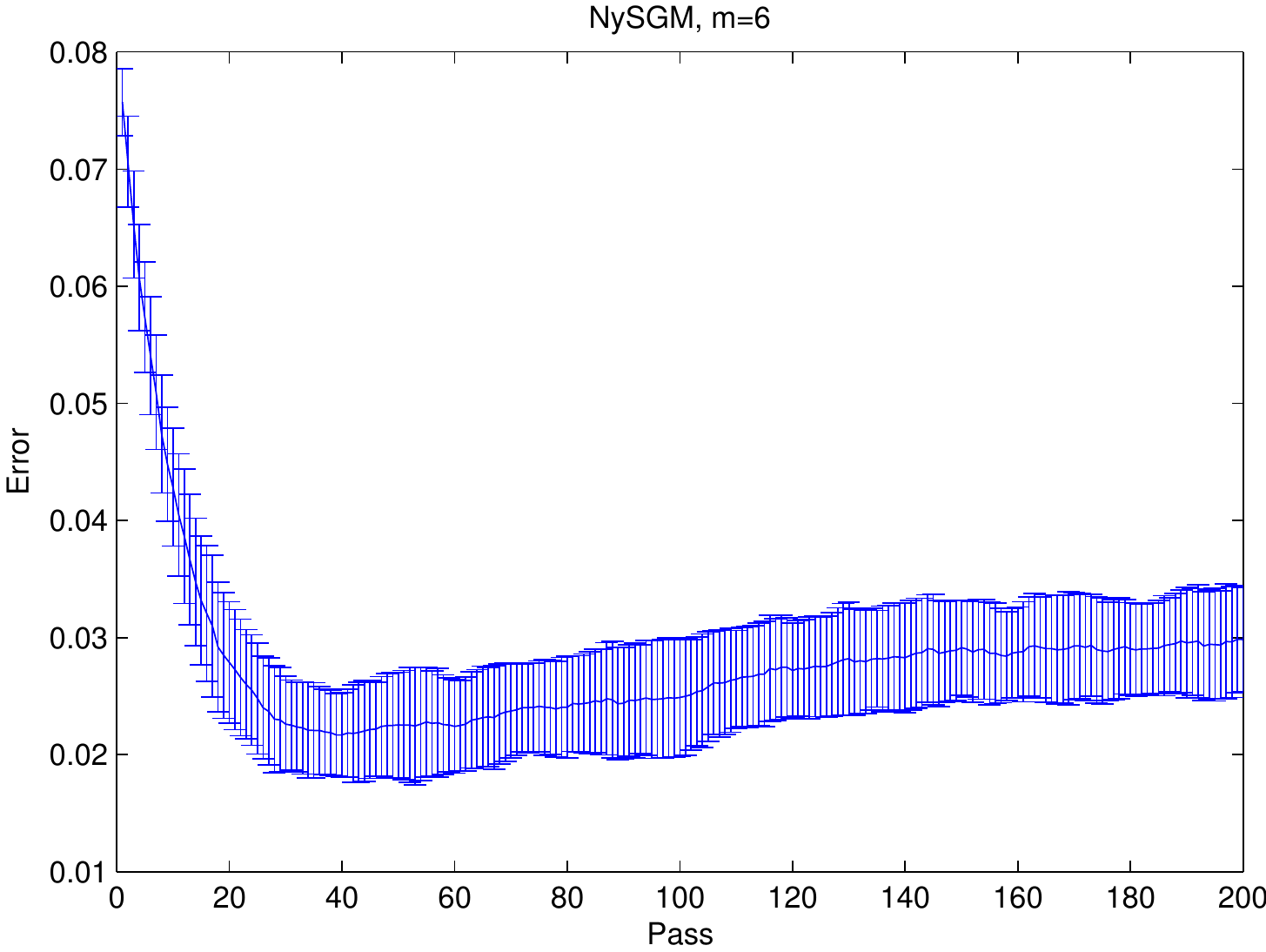}
    \end{subfigure}
     ~
    \begin{subfigure}[]
    {0.3\textwidth}
    \includegraphics[width=\textwidth]{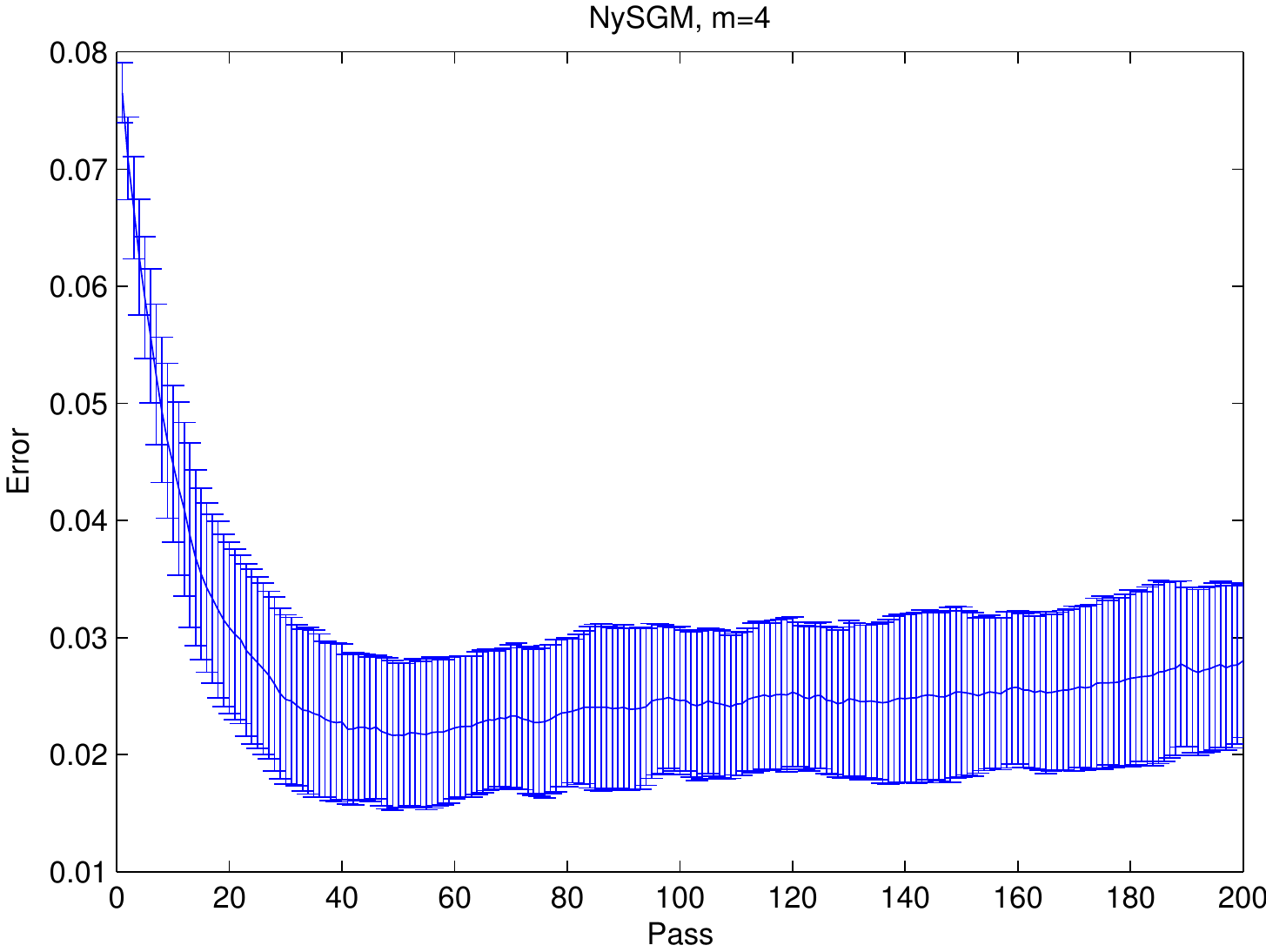}
    \end{subfigure}
    ~
    \begin{subfigure}[]
    {0.3\textwidth}
    \includegraphics[width=\textwidth]{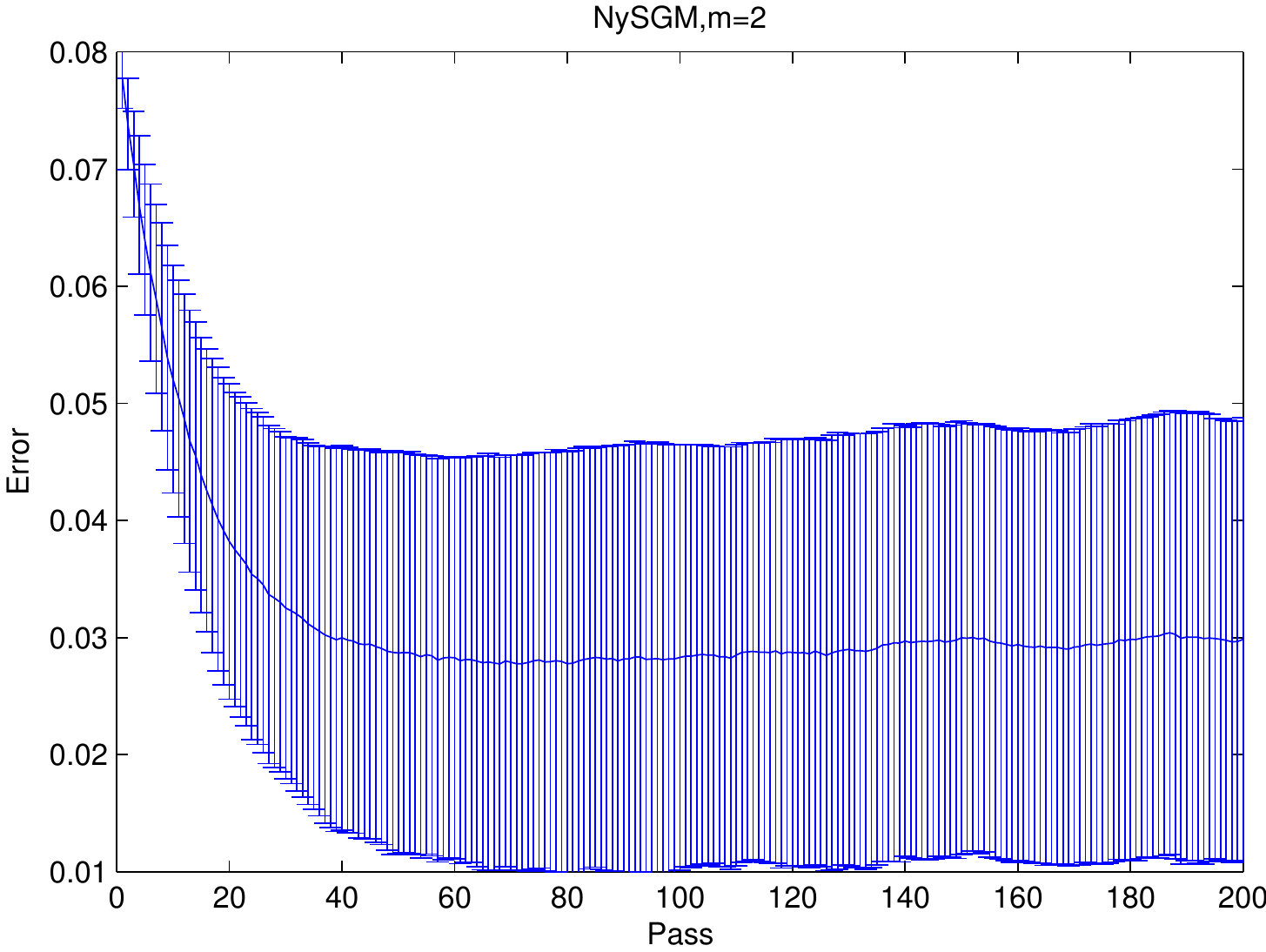}
    \end{subfigure}
    \caption{Approximated Generalization Errors for NySGM on {\em toy data}, with Different Subsampling Level $m=\{2,4,6,8,10,12\}$.}
    \label{fig:1}
\end{figure}

\section{Generalization Error Bounds for NySGM}\label{sec:main}
In this section, we present our main results on generalization errors for NySGM, followed by some simple discussions. Throughout this paper, we make the following basic assumptions.
\begin{as}\label{as:basic}
$\HK$ is separable, $K$ is measurable and
  there exists a constant $\kappa \in [1,\infty[$, such that for all $x \in X,$
\be\label{eq:HK}
  K(x,x) \leq \kappa^2 .
\ee
Furthermore, Problem \eref{erm} has at least a solution $\FH \in \HK$.
\end{as}
 The boundedness assumption \eref{eq:HK} is fairly common in standard learning theory. It can be satisfied, for example when the kernel is a Gaussian kernel.
 The condition on the existence of at least one minimizer in $\HK$ is for the sake of easy presentation. Such a condition can be relaxed, by using a more involved analysis as that in \cite{lin2016optimal}.

 Under these basic assumptions, we can state our first theorem as follows. It provides generalization error bounds for the studied algorithms with different choices of the step-size, the mini-batch size and the total number of iterations.

\begin{thm}\label{thm:simple}
  Let $|y| \leq M$ a.s., $\delta \in (0,1)$,
$n \gtrsim 1$ and $m \gtrsim \sqrt{n} \log n$. Consider Algorithm \ref{alg:1} with either of the following choices on $\eta_t$, $b$ and $T_*$:\\
  I) $\eta_t \simeq (\log n)^{-1}$ for all $t \in [T_*]$, $b=T_* = \lceil \sqrt{n} \rceil;$\\
  II)  $\eta_t \simeq n^{-1/2}$ for all $t \in [T_*]$, $b=1, T_* = n.$\\
  Then with probability at least $1-\delta,$
   \be\label{eq:errBound}
\mE_{\J}[\mcE(f_{T_*+1})]  - \mcE(\FH)
\lesssim  n^{-1/2} \log n.
\ee
Here, we use the notations $a_1 \lesssim a_2$ to mean $a_1 \leq C a_2$
 for some positive constant $C$ which is depending only (a polynomial function)
 on $\kappa, M, \|\TK\|, \|\FH\|_{\HK}$ and $\log{1\over \delta},$ and $a_1 \simeq a_2$  to mean $a_2 \lesssim a_1 \lesssim a_2$.
\end{thm}

We add some comments on the above results. First, the bounded output assumption is  trivially satisfied for some learning problems such as binary classification problems where $Y = \{-1,1\}.$ Second, the error bound in \eref{eq:errBound} is optimal up to a logarithmic factor, in the sense that it matches the minimax rate in \cite{caponnetto2007optimal} and  those of kernel ridge regression \cite{smale2007learning,caponnetto2007optimal,steinwart2009optimal}.
Moreover, according to Theorem \ref{thm:simple}, NySGM with two different choices on the step-size and the mini-batch size can achieve optimal learning error bounds after one pass over the data, provided that the subsampling level $m \simeq \sqrt{n}$. Thus, if the computer computes and stores $\mbR$ and $\mbK_{\mbtx \mbx}$ in the preprocessing and then updates $\mbb_t$ by \eref{eq:algNumRea} based on $\mbK_{\mbtx \mbx},\mbR$ and  ${\bf y},$ according to \eref{eq:cost1}, the cost
for NySGM with both (I) and (II) is $O(n^{1.5})$ in memory and $O(n^{1.5}d)$ in time, lower than $O(n^{1.5})$ in memory and $O(n^{1.5}d + n^2)$ in time required by Nystr\"{o}m KRR \cite{rudi2015less}. Alternatively, if the computer computes and stores $\mbR$  in the preprocessing and then updates $\mbb_t$ by \eref{eq:algNumRea} based on $\mbR$ and  ${\bf z},$
 the cost  is $O(nd)$ in memory and $O(n^{1.5}d + n^2)$ in time for NySGM (II), while $O(nd)$ in memory  and $O(n^{1.5}d)$ in time for NySGM (I).
 Compared to $O(nd)$ in memory and $O(n^{2}d)$ in time for classic SGM, NySGM using mini-batches has lower computational cost. In this sense, using mini-batches can reduce the computational complexity. Finally, using mini-batches allows using a larger step-size, while achieving the same optimal error bounds.

Theorem \ref{thm:simple} provides generalization error bounds for the studied algorithm, without considering the possible effect of benign assumptions on the problem. In the next theorem,  we will show that when the learning problem satisfies some additional regularity and capacity assumptions, it is possible to achieve faster learning rates than $O(n^{-1/2})$.
Also, the boundedness assumption on the output in Theorem \ref{thm:simple} will be replaced by a less strict condition, the moment hypothesis on $|y|^2$ as follows.

\begin{as}\label{as:noiseExp}
  There exists constants $M \in ]0,\infty[$ and $\nu \in ]1,\infty[$ such that
  \be\label{noiseExp}
  \int_{Y} y^{2l} d\rho(y|x) \leq l! M^l \nu, \quad \forall l \in \mN,
  \ee
  $\rho_{ X}$-almost surely.
\end{as}

To present our next assumptions, we introduce the covariance operator $\TK: \HK \to \HK$, defined by $\TK = \int_X \la \cdot, K_x\ra K_x d\rho_{X}.$ Under Condition \eref{eq:HK}, $\TK$ is known to be positive definite and trace class.
Thus, $\TK^{\zeta}$ with $\zeta \in \mR$ can be defined by using the spectral theory.
We make the following assumption on the regularity of the target function $\FH$.
\begin{as}\label{as:regularity}
 For some $\zeta \geq 0$ and $R>0$, $\|\TK^{-\zeta} \FH \|_{\HK} \leq R.$
\end{as}

The above assumption is very standard \cite{cucker2007learning,rosasco2015learning} in nonparametric regression. It characterizes how big the subspace that the target function $\FH$ lies in. Particularly, the bigger the $\zeta$ is,
the more stringent the assumption is and the smaller the subspace is, since $\TK^{\zeta_1}(\HK) \subseteq \TK^{\zeta_2}(\HK)$ when $\zeta_1 \geq \zeta_2.$ Moreover, when $\zeta =0,$ we are making no assumption.

The last assumption relates to the capacity of the hypothesis space.
\begin{as}\label{as:eigenvalues}
  For some $\gamma \in [0,1]$ and $c_{\gamma}>0$, $\TK$ satisfies
\be\label{eigenvalue_decay}
  \tr ( \TK (\TK + \lambda I)^{-1} ) \leq c_{\gamma} \lambda^{-\gamma}, \quad \mbox{for all } \lambda>0.
\ee
\end{as}
The left hand-side of \eref{eigenvalue_decay} is called as the effective
dimension \cite{caponnetto2007optimal}, or the degrees of freedom \cite{zhang2005learning}.
It can be related to covering/entropy number conditions, see \cite{steinwart2008support} for further details.
Assumption \ref{as:eigenvalues} is always true for $\gamma=1$ and $c_{\gamma} =\kappa^2$, since
 $\TK$ is a trace class operator which implies the eigenvalues of $\TK$, denoted as $\sigma_i$, satisfy
 $\tr(\TK) = \sum_{i} \sigma_i \leq \kappa^2.$
  This is referred to as the capacity independent setting.
  Assumption \ref{as:eigenvalues} with $\gamma \in]0,1]$ allows to derive better error rates. It is satisfied, e.g.,
   if the eigenvalues of $\TK$ satisfy a polynomial decaying condition $\sigma_i \sim i^{-1/\gamma}$, or with $\gamma=0$ if $\TK$ is finite rank.

Now, we are ready to state our next theorem as follows.

\begin{thm}\label{thm:simp}
Under Assumptions \ref{as:noiseExp}, \ref{as:regularity} and \ref{as:eigenvalues}, let $\zeta\leq 1/2$, $\delta \in (0,1)$,
$n \gtrsim 1$ and
\be\label{eq:subsamp}
m \gtrsim n^{{1\over 2\zeta+\gamma+1}} \log n.
\ee
Set $\eta_t = \eta $ for all $t \in [T],$ $0<\eta \lesssim {1 \over \log T},$ then the following holds with probability at least $1-\delta:$
\be\label{eq:genBoundB}
\begin{split}
\mE_{\J}[\mcE(f_{t+1})]  - \mcE(\FH)
\lesssim (\eta t)^{-(2\zeta+1)} + (\eta t)^2 n^{-{2\zeta+3 \over 2\zeta+ \gamma+1}}
+  \left(1 \vee n^{-{1\over 2\zeta+\gamma+1}}\eta t \right) \eta b^{-1} \log (T).
\end{split}
\ee
and particularly, if $T_* \simeq n^{{1\over 2\zeta+\gamma+1}}\eta^{-1},$
  \be\label{eq:genBoundC}
\mE_{\J}[\mcE(f_{T_*+1})]  - \mcE(\FH) \lesssim  n^{-{2\zeta+1 \over 2\zeta+ \gamma+1}} + \eta b^{-1} \log n .
\ee
Here, we use the notation $a_1 \lesssim a_2$ to mean $a_1 \leq C a_2$
 for some positive constant $C$ which is depending only (a polynomial function)
 on $\kappa,c_{\gamma}, M, v,\zeta,\gamma, \|\TK\|, R$ and $\log{1\over \delta},$
 and is independent of $\eta, m, n,b.$
\end{thm}

The above result is a direct consequence of Theorem \ref{thm:comp} in the coming Subsection \ref{subsec:der}, from which, the interested readers can also find convergence results for the decaying step-size setting, i.e., $\eta_t = \eta t^{-\theta}$ with $\theta\in(0,1)$, as well as the omitted constants. There are three terms in the upper bounds of  \eref{eq:genBoundB}. The first two terms are related to the regularity of the target function and the sample size, and they arise from  estimations of the projected bias and the sample variance. The last term results from estimating the computational variance due to the random choices of the points.
Note that there is a trade-off between the first two terms. Stopping too earlier may lead to a large projected bias, while stopping too late may enlarge the sample variance. The optimal number of iterations $T_*$ is thus achieved by balancing these two terms.
Furthermore, to achieve optimal rates, it is necessary to choose a suitable step-size $\eta$ and a mini-batch size $b$ such that the computational variance is smaller than the first term of \eref{eq:genBoundC}. In the next corollary, we provide different choices on step-size and mini-batch size to achieve optimal convergence rates.

\begin{corollary}\label{cor:1}
Under Assumptions \ref{as:noiseExp}, \ref{as:regularity} and \ref{as:eigenvalues}, let $\zeta\leq 1/2$, $\delta \in (0,1)$,
$n \gtrsim 1$ and \eref{eq:subsamp}.
Consider Algorithm \ref{alg:1} with any of the following choices on $\eta_t,$ $b$ and $T_*$:\\
I) $\eta_t \simeq n^{-{2\zeta+1 \over 2\zeta+\gamma+1} }$, $b=1$ and $T_* \simeq n^{{2\zeta + 2\over 2\zeta+\gamma+1}}$;\\
II) $\eta_t \simeq (\log n)^{-1}$, $b\simeq n^{{2\zeta+1 \over 2\zeta+\gamma+1} }$ and $T_* \simeq n^{{1\over 2\zeta+\gamma+1}}\log n$.\\
III) $\eta_t \simeq n^{-1}$, $b=1$ and $T_* \simeq n^{{1\over 2\zeta+\gamma+1}} n$;\\
IV) $\eta_t \simeq n^{-1/2}$, $b\simeq \sqrt{n}$ and $T_* \simeq n^{{1\over 2\zeta+\gamma+1}} \sqrt{n}$;\\
Then the following holds with probability at least $1-\delta:$
  \be\label{eq:genErr}
\mE_{\J}[\mcE(f_{T_*+1})]  - \mcE(\FH)
\lesssim  n^{-{2\zeta+1 \over 2\zeta+ \gamma+1}} \log n.
\ee
\end{corollary}
We add some comments.
First, the convergence rate in \eref{eq:genErr} is optimal up to a logarithmic factor, as it matches the minimax rate in \cite{caponnetto2007optimal}.
Thus, with a subsampling level $m\simeq n^{{1\over 2\zeta+\gamma+1}} \log n$, NySGM  with suitable choices of step-size, mini-batch size and number of iterations/passes can generalize optimally.
Second, different choices of step-size, mini-batch size and number of iterations/passes correspond to different regularization regimes. Particularly,
in the last two regimes, the step-size and the mini-batch size are fixed as some universal constants, while the number of iterations/passes is depending on the unknown distribution parameters $\gamma$ and $\zeta$.
In this case, the only regularization parameter is the number of iterations/passes, which can be tuned by cross-validation in practice.
Besides, the step-size and the number of iterations/passes in the first regime, or the mini-batch size and the number of iterations/passes in the second regime, depend on the unknown distribution parameters, and they can be tuned by cross-validation in practice. Third, according to Corollary \ref{cor:1},  the number of passes needed for NySGM to generalize optimally  is $n^{{1\over 2\zeta+\gamma+1}}$ in the last two regimes, while $n^{{1-\gamma\over 2\zeta+\gamma+1}}$ in the first two regimes. In comparison, NySGM with the first two regimes has a smaller number of passes than that of the last two regimes.
This indicates that NySGM with the first two regimes may have some certain advantage on computational complexity,
although in this case the step-size or the mini-batch size might need to be tuned.


The next corollary is a direct consequence of Corollary \ref{cor:1} in the capacity independent case. Indeed, in the capacity independent case, as mentioned before,
Assumption \ref{as:eigenvalues} is always satisfied with $\gamma=1$. Thus, following from Corollary \ref{cor:1}, we have the following results.

\begin{corollary}\label{cor:2}
Under Assumptions \ref{as:noiseExp} and \ref{as:regularity}, let $\zeta\leq 1/2$, $\delta \in (0,1)$,
$n \gtrsim 1$ and
\bea
m \gtrsim n^{1 \over 2\zeta+2} \log n.
\eea
Consider Algorithm \ref{alg:1} with any of the following choices on $\eta_t,$ $b$ and $T_*$:\\
I) $\eta_t \simeq n^{-{2\zeta+1 \over 2\zeta+2} }$, $b=1$ and $T_* \simeq n$;\\
II) $\eta_t \simeq (\log n)^{-1}$, $b\simeq n^{{2\zeta+1 \over 2\zeta+2} }$ and $T_* \simeq n^{1 \over 2\zeta+2} \log n$.\\
Then with probability at least $1-\delta,$
\bea
\mE_{\J}[\mcE(f_{T_*+1})] - \mcE(\FH) \lesssim n^{-{2\zeta+1 \over 2\zeta+2}} \log n.
\eea
\end{corollary}

From the above corollary, we see that NySGM achieves optimal capacity-independent rate with one pass over the data if the step-size or the mini-batch size is suitably chosen.
Theorem \ref{thm:simple} is a direct consequence of Corollary \ref{cor:2} in the special case that $\zeta=0$ and $|y| \leq M$ a.s.. We finish this section with some remarks.
\begin{rem}\label{rem:theorem}
  I) All results in Theorem \ref{thm:simp} and its corollaries still hold when the condition on the subsampling level \eref{eq:subsamp} is replaced by
  $ m \gtrsim \mcN_{\infty}(\|\TK\|n^{-{1\over 2\zeta+\gamma+1}}),$ where $\mcN_{\infty}(\lambda) = \sup_{x\in X}\la (\TK+\lambda)^{-1}K_x, K_x\ra_{\HK}.$
  Thus, if \footnote{Note that this condition is always satisfied with $\gamma'=1$ as by \eref{eq:HK},
  $\mcN_{\infty}(\lambda) \leq \kappa^2/\lambda$.} $\mcN_{\infty}(\lambda) \lesssim \lambda^{-\gamma'}$ for all $0<\lambda\lesssim 1,$ for some $\gamma'\in [0,1]$, then the condition \eref{eq:subsamp}
  can be replaced by $m \gtrsim n^{\gamma' \over 2\zeta+\gamma+1}\log n.$\\
  II) If we consider a more involved subsampling technique, the approximate leverage scores Nystr\"{o}m methods \cite{drineas2012fast,cohen2015uniform,el2014fast}, as we will see in our proof,
  all results in Theorem \ref{thm:simp} and its corollaries still hold with a less strict requirement on the subsampling level,
  $m \gtrsim n^{\gamma \over 2\zeta+\gamma+1} \log n.$
\end{rem}

\section{Discussions}\label{sec:dis}
We must compare our results with related works.
There is a large amount of work on online learning (OL) algorithms
and, more generally, stochastic approximations, see, e.g., \cite{robbins1951stochastic,polyak1992acceleration,cesa-bianchi2004} and the references
therein. Here, we briefly review some recent works on online learning algorithms in the framework of nonparametric regression with the square loss. In what follows, we used the term ``online learning algorithm" to mean the ``stochastic gradient method'' that each sample point can only be used once.
First, OL with regularization has been studied in \cite{yao2006dynamic,tarres2014online}, where the recursion appears as
\be\label{eq:ol}
g_{t+1} = g_t - \eta_t \left( (g_{t}(x_t) - y_t) K_{x_t} + \lambda_t g_t\right), \quad t=1,\cdots, n.
\ee
Here,  $\lambda_t >0 $ is a regularization parameter.
In particular, generalization error bounds of order $O(n^{-{2\zeta +1 \over 2\zeta+2}})$ in confidence were proved  in \cite{tarres2014online} for OL with suitable choices of  regularization parameter $\lambda_t$ and step-size $\eta_t$, when $\zeta\leq 1/2$, without considering the capacity assumption, i.e.,$\gamma=1$.
Comparing with our results for NySGM from Corollary \ref{cor:1}.(II), as indicated in Table \ref{table:1}, the computational cost of NySGM is lower, while both algorithms have the same optimal rates in the capacity independent case.
Second, \cite{zhang2004solving,ying2008online} studied unregularized OL, i.e., \eref{eq:ol} with $\lambda_t = 0$, where the derived convergence rates \cite{ying2008online} are of order $O(n^{-{2\zeta +1 \over 2\zeta+2}})$ and are in expectation, without considering the capacity  assumption.
If we make no assumption on the capacity condition, the derived rate for NySGM in this paper is the same as that of \cite{ying2008online} for unregularized OL for the case $\zeta\leq 1/2,$ and the former has a lower computational cost.
Note that NySGM saturates for $\zeta\geq 1/2$, while OL from \cite{ying2008online} does not. We conjecture that by considering a different subsampling technique, we may get optimal error bounds even for $\zeta \geq 1/2$.
Third, by considering an averaging scheme, optimal capacity dependent rates can be proved for OL with appropriate parameters, either without \cite{dieuleveut2014non} or with \cite{dieuleveut2016harder} regularization.  In comparisons of NySGM with averaged OL (AveOL) from \cite{dieuleveut2014non},
as indicated in Table \ref{table:1}, both algorithms have the same optimal rates.
When $2\zeta + \gamma \geq 1,$ the storage complexities for both algorithms are of the same orders, while the computational complexity for the former is lower than that of the latter. For the case $2\zeta + \gamma \leq 1,$ NySGM seemly has higher storage requirement.
But as we will see in Subsection \ref{subsec:der}, by considering the approximate leverage scores (ALS) Nystr\"{o}m methods or making an extra assumption on the input data,
the subsampling level for NySGM can be further reduced, which thus potentially leads to a smaller computational complexity.
However, the ALS subsampling technique, or the extra condition is less well understood and should be further studied in the future.
Finally, \cite{lin2016optimal} studied (multi-pass) SGM, i.e, Algorithm \ref{alg:1} without projection.
With suitable parameter choices, SGM achieves optimal rate after some number of iterations/passes \cite{lin2016optimal}.
In comparisons, again, the computational complexity for NySGM  is lower than that of SGM from \cite{lin2016optimal} when $2\zeta+\gamma\geq 1$.
All results mentioned in the above are summarized in Table \ref{table:1}.
 \begin{table}
 \newcommand{\tabincell}[2]{\begin{tabular}{@{}#1@{}}#2\end{tabular}}
 \centering
 \resizebox{\textwidth}{!}{
 \begin{tabular}{ | c | c | c | c | c | c | c | c | c|}
  \hline
  Alg & $m$ & Ass. & $\lambda_t$ & $\eta_t$ & $b$ & $T$ & Rate & \tabincell{c}{Memory\\ \& Time} \\
  \hline
  \tabincell{c}{NySGM \\ (This paper, \\ Corollary \ref{cor:1} (II))} & $n^{{1 \over 2\zeta+\gamma+1}}$ & \tabincell{c}{$\gamma\leq 1$\\ $\zeta\leq {1\over 2}$} & $0$  & $1$ & $n^{{2\zeta+1 \over 2\zeta+\gamma+1}}$ & $n^{{1 \over 2\zeta+\gamma+1}}$ & $n^{-{2\zeta + 1 \over 2\zeta+\gamma+1}}$ & \tabincell{c}{$n^{{2 \over 2\zeta+\gamma+1}} +¡¡nd$ \\ $n^{{2\zeta+3 \over 2\zeta+\gamma+1}} d$}  \\
  \hline
  OL \cite{zhang2004solving} & $n$ & \tabincell{c}{$\gamma= 1$ \\$\zeta =0$} & $0$ & $n^{-{1 \over 2}}$  &  $1$ & $n$ & $n^{-{1 \over 2}}$ & \tabincell{c}{$n + nd$\\ $n^2 d$}  \\
  \hline
  OL \cite{tarres2014online}& $n$ & \tabincell{c}{$\gamma= 1$ \\$\zeta \leq {1\over 2}$} & $n^{-{1 \over 2\zeta+2}}$  & $n^{-{2\zeta+1 \over 2\zeta+2}}$  &  $1$ & $n$ & $n^{-{2\zeta + 1 \over 2\zeta+2}}$ & \tabincell{c}{$n + nd$\\ $n^2 d$}  \\
  \hline
  OL \cite{ying2008online} & $n$ & \tabincell{c}{$\gamma= 1$ \\$\zeta < \infty$} & $0$ & $n^{-{2\zeta +1 \over 2\zeta +2}}$ & $1$ & $n$ & $n^{-{2\zeta + 1 \over 2\zeta+2}}$ & \tabincell{c}{$n + nd$\\ $n^2 d$} \\
  \hline
   AveOL \cite{dieuleveut2014non} & $n$ & \tabincell{c}{$\gamma\leq 1$ \\$\zeta \leq {1\over 2}$} & $0$  & $n^{-{2\zeta+\gamma \over 2\zeta+\gamma+1}}$  &  $1$ & $n$ & $n^{-{2\zeta + 1 \over 2\zeta+\gamma+1}}$ & \tabincell{c}{$n + nd$\\ $n^2 d$}  \\
  \hline
  SGM \cite{lin2016optimal}  & $n$ & \tabincell{c}{$\gamma\leq 1$ \\$\zeta < \infty$} & $0$ & $n^{-{2\zeta +1 \over 2\zeta +1+\gamma}}$ & $1$ & $n^{2\zeta+2 \over 2\zeta+1+\gamma}$ & $n^{-{2\zeta + 1 \over 2\zeta+\gamma+1}}$ & \tabincell{c}{$n + nd$\\ $n^{4\zeta+3+\gamma \over 2\zeta+\gamma+1} d$} \\ \hline
\end{tabular}}
\caption{}\label{table:1}
{\it Summary of assumptions and results for NySGM and related approaches including online learning (OL),
averaged OL (AveOL) and SGM. Noted that all the logarithmic factors are ignored.}
\end{table}

Next, we will briefly review some of the recent theoretical results on Nystr\"{o}m subsampling.
Theoretical results considering the discrepancy between a given empirical kernel matrix and its subsampled version can be found in, e.g., \cite{gittens2013revisiting,drineas2012fast,kumar2012sampling} and references therein.
While interesting in their own right, these latter results do not directly yield information on the generalization properties of the obtained algorithm.
Results on this direction were first derived in e.g., \cite{cortes2010impact,jin2013improved} with the fixed design regression setting and in \cite{rudi2015less} with the random design regression setting, both for Nyst\"{o}m KRR.
Particularly, a sharp learning rate of order $O(n^{-{2\zeta+1\over 2\zeta+\gamma+1}})$ was derived in \cite{rudi2015less} for Nystr\"{o}m KRR provided that the subsampling level $m \gtrsim n^{{1\over 2\zeta+\gamma+1}}$.
In comparison, both Nystr\"{o}m KRR and NySGM have the same requirement on subsampling level while sharing the same optimal rates.  According to Corollary \ref{cor:1}.(II) and Equation \eref{eq:cost1}, the cost for NySGM are $O(n^{2\zeta+\gamma+2 \over 2\zeta+\gamma+1})$ in memory and $O(n^{2\zeta+3 \over 2\zeta+\gamma+1})$ in time,
compared to $O(n^{2\zeta+\gamma+2 \over 2\zeta+\gamma+1})$ in memory and $O(n^{2\zeta+3 +\gamma \over 2\zeta+\gamma+1})$ in time of Nystr\"{o}m KRR from \cite{rudi2015less}.
The most related to our work is \cite{lu2016large}, where a regularized OL with Nystr\"{o}m is investigated. Certain convergence results on regret errors with respect to the KRR estimator were shown in \cite{lu2016large} under a bounded assumption on the gradient estimates, but the generalization properties are less clear. Moreover, the derived rates are capacity-independent and
both the derived rates and the subsampling level tend to be suboptimal, as the error bounds are depending directly on the discrepancy between a given empirical kernel matrix and its subsampled version.

Note that in this paper, we assume that the parameter choices on step-size, mini-batch size, number of iterations and subsampling level,
involved in our theoretical results can be given in advance, and we did not consider model selection of these parameters. In practice, model selection on these parameters can be possibly realized by a cross-validation approach \cite{steinwart2008support,caponnetto2010cross},
and one can possibly prove that such an approach can lead to the same statistical results.
Indeed, we will introduce the cross-validation approach for tuning the step-size in Subsection \ref{subsec:adaptive},
and prove theoretical results for such an approach. For model selection on the other parameters, we left it as an open problem in the future.

We end up this section with some remarks and future issues.
First, in this paper, all derived convergence results for NySGM hold for the case that Problem \eref{erm} has at least one solution $\FH \in \HK$.
In the case that the condition is not satisfied, one can possibly derive similar results as those in \cite{lin2016optimal},
by a more involved argument. Second, in this paper, we only prove results on convergence in  $\LR$-norm,
but the extension to results in $\HK$-norm, and moreover the `middle' norm between $\HK$-norm and $\LR$, are possible.
Third, all results in this paper are stated for a real-valued output space case, but they can be easily extended to the case that the output space is a general Hilbert space,
as those in \cite{caponnetto2007optimal}. Fourth, we didn't try to optimize the conditions and error bounds, some of which can be further improved by a more involved argument. In particular, the boundedness assumption  \eref{eq:HK} can be possibly replaced by $\int_{X}K(x,x)d\rho_X \leq \kappa^2.$ Also, the logarithmic factor from the derived error bounds can be possibly removed either using a more involved argument or considering
a proper non-uniform averaging scheme as that in \cite{shamir2013stochastic}.
Finally, using the techniques developed in this paper, it would be interesting to study stochastic gradient methods with a preconditioned operator as that in \cite{ma2017diving}, or random features \cite{rahimi2007random}.
Also, rather than considering a simple stochastic gradient method, it would be interesting to consider more sophisticated, ¡®accelerated¡¯
iterations \cite{schmidt2013minimizing}, and assess the potential advantages in terms of computational and generalization aspects.

\section{Proofs}\label{sec:proofs}
This section is devoted to the proof of all related equations and results stated in the last sections. We begin in the next subsection with the basic notations.

\subsection{Notation}

Denote $\rho_X(\cdot)$ as the induced marginal measure of $\rho$ on $X$, and  $\rho(\cdot | x)$ as the conditional probability measure on $\mR$ with respect to $x \in X$ and $\rho$. The function minimizing the expected risk over all measurable functions is the regression function, which is given by
\be\label{regressionfunc}
f_{\rho}(x) = \int_Y y d \rho(y | x),\qquad x \in X.
\ee
The Hilbert spaces of square integral functions with respect to $\rho_X$, with its induced norm given by
$\|f\|_{\rho} = \|f\|_{\LR} = \left (\int_X |f(x)|^2 d \rho_X\right)^{1/2}$, is denoted by $(\LR,\|\cdot\|_{\rho}).$
Under Assumption \ref{as:basic}, we know that the projection of the regression function $f_{\rho}$ onto the closer of $\HK$ in $\LR$, lies in $\IK(\HK)$, and $\FH$ is a solution of the normalized embedding equation
\be\label{eq:embeddingEqua}
\TK \FH = \IK^* f_{\rho}.
\ee
For any  $f\in \HK$ and $x\in X$, the following well known reproducing property holds:
\be\label{eq:reproduce}
\la f, K_x\ra = f(x).
\ee

For any $t\in \mR,$ the set $\{1, ..., t\}$ of the first $t$ positive integers is denoted by $[t]$.
$\Pi_{t+1}^T(L) = \prod_{k=t+1}^T (I - \eta_k L)$ for $t \in [T-1]$ and $\Pi_{T+1}^T(L) = I,$ for any operator $L: H \to H,$
where $H$ is a Hilbert space and $I$ denotes the identity operator on $H$.
$\mE[\xi]$ denotes the expectation of a random variable $\xi.$
For a given bounded operator $L: \HK \to \HK', $ $\|L\|$ denotes the operator norm of $L$, i.e., $\|L\| = \sup_{f\in \HK, \|f\|_{\HK}=1} \|Lf\|_{\HK'}$.
We will use the conventional notations on summation and production: $\prod_{i=t+1}^t = 1$ and $\sum_{i=t+1}^t = 0.$
$\|\cdot\|_{\infty}$ denotes the supreme norm with
respect to $\rho_X.$

We introduce the inclusion operator $\IK: \HK \to \LR$, which is continuous under Assumption \eref{eq:HK}. Furthermore, we consider the adjoint operator $\IK^*: \LR \to \HK$, the covariance operator $\TK: \HK \to \HK$ given by $\TK = \IK^* \IK$, and the operator $\LK : \LR \to \LR$ given by $\IK \IK^*.$
It can be easily proved that $ \IK^*f = \int_X K_x f(x) d\rho_X(x)$
and $\TK = \int_X \la \cdot , K_{x} \ra_{\HK}K_x d \rho_X(x).$
The operators $\TK$ and $\LK$ can be proved to be positive trace class operators (and hence compact).
 For any function $f \in \HK$,
the $\HK$-norm can be related to the $\LR$-norm by $\sqrt{\TK}:$ \cite{bauer2007regularization}
\be\label{isometry}
\|f\|_{\rho} = \|\IK f\|_{\rho} = \left\|\sqrt{\TK} f\right\|_{\HK}.
\ee

We define the sampling operator (with respect to $\mbx$) $\SX: \HK \to \mR^n$ by $(\SX f)_i = f(x_i) = \la f, K_{x_i} \ra_{\HK},$ $i \in [n]$, where the norm $\|\cdot\|_{\mR^n}$ is the standard Euclidean norm.
Its adjoint operator $\SX^*: \mR^n \to \HK,$ defined by $\la \SX^*{\bf y}, f \ra_{\HK} = \la {\bf y}, \SX f\ra_{\mR^n}$ for ${\bf y} \in \mR^n$ is thus given by
\be\label{eq:sampleOperAdjoint}
\SX^*{\bf y} = \sum_{i=1}^n y_i K_{x_i}.
\ee
Moreover, we can define the empirical covariance operator (with respect to $\mbx$)  $\TX: \HK \to \HK$ such that $\TX = {1 \over n}\SX^* \SX$. Obviously,
\bea
\TX = {1 \over n} \sum_{i=1}^n \la \cdot, K_{x_i} \ra K_{x_i}.
\eea
Finally, we can define the sampling and empirical covariance operators with respect to any given set $\hat{\bf x} \subset X,$ $|\hat{\bf x}| \in \mN$, in a similar way.
For any finite subsets $\mbx$ and $\mbx'$ in $X$, denote the $|{\mbx}| \times |\mbx'|$ kernel matrix $[K(x,x')]_{x\in \mbx,x'\in\mbx'}$ by
$\mbK_{\mbx \mbx'}$. Obviously,
\be\label{eq:kernelxx}
\mbK_{\mbx \mbx'} = \mcS_{\mbx} \mcS_{\mbx'}^*.
\ee
For notational simplicity, we let $ \TKL = \TK + \lambda I $ and
$\TXL = \TX + \lambda I$ for any $\lambda>0.$

Let $\STX = U \Sigma V^*$ be the SVD of $\STX$, where $U: \mR^{t} \to \mR^m,$ $\Sigma: \mR^t \to \mR^t,$ $V: \mR^t \to \HK$, $t \leq m$ and $\Sigma= \diag(\sigma_1,\sigma_2,\cdots,\sigma_t)$ with $\sigma_1 \geq \cdots \geq \sigma_t >0,$ $U^* U = \mathbf{I}_t$ and $V^*V= \mathbf{I}_t$. Then the orthogonal projection operator $\Ptx$ is given by
 \be\label{eq:projOper}
 \Ptx = V V^* = \mcS_{\mbtx}^* (\mcS_{\mbtx} \mcS_{\mbtx}^*)^{\dag} \mcS_{\mbtx}
 \ee

For any $\lambda>0$, we define the random variable $\mcN_{x}(\lambda) = \la K_x, (\TK + \lambda I)^{-1} K_x \ra_{\HK}$ with $x \in X$ distributed according to $\rho_X$ and let
\bea
\mcN(\lambda) = \mE[\mcN_x(\lambda)], \qquad \mcN_{\infty}(\lambda) = \sup_{x \in X} \mcN_x(\lambda).
\eea

\subsection{Equivalent Forms of NySGM}\label{equivalent}
In this subsection, we prove that Algorithm \ref{alg:1} is equivalent to \eref{eq:algNumRea}.

Note that by the first equality of  \eref{eq:algNumRea} and \eref{eq:sampleOperAdjoint}, $f_{t} = \STX^* \mbR \mathbf{b}_t.$
Combining with \eref{eq:reproduce} and \eref{eq:kernelxx}, for any $x \in X$, $f_{t}(x) = \mcS_{x} f_t = \mcS_x \STX^* \mbR \mathbf{b}_t = \mbK_{x\mbtx} \mbR \mathbf{b}_{t} = \mbK_{\mbtx x}^{\top} \mbR \mathbf{b}_{t}.$
Thus, $\mbK_{\mbtx x_{j_i} }^{\top} \mbR \mbb_t = f_t(x_{j_i})$ and following from the second equality of \eref{eq:algNumRea},
\bea
\mathbf{b}_{t+1} = \mathbf{b}_t - \eta_t {1 \over b}\sum_{i= b(t-1)+1}^{bt} (f_t(x_{j_i}) - y_{j_i}) \mbR^{\top} \mbK_{\mbtx x_{j_i}},
\eea
which leads to
\bea
\mbK_{x\mbtx} \mbR \mathbf{b}_{t+1} = \mbK_{x\mbtx} \mbR\mathbf{b}_t - \eta_t {1 \over b}\sum_{i= b(t-1)+1}^{bt} (f_t(x_{j_i}) - y_{j_i}) \mbK_{x\mbtx} \mbR \mbR^{\top} \mbK_{\mbtx x_{j_i}}.
\eea
Note that from the definition of $\mbR$, and following from \eref{eq:kernelxx} and \eref{eq:projOper},
 $\mbK_{x\mbtx} \mbR \mbR^{\top} \mbK_{\mbtx x_{j_i}} = \mbK_{x\mbtx} \mbK_{\mbtx \mbtx}^{\dag} \mbK_{\mbtx x_{j_i}} =
 \mcS_{x} \mcS_{\mbtx}^* (\mcS_{\mbtx} \mcS_{\mbtx}^*)^{\dag} \mcS_{\mbtx} K_{x_{j_i}} = \mcS_{x} \Ptx K_{x_{j_i}} = (\Ptx K_{x_{j_i}})(x).$
 We thus have
 \bea
 f_{t+1}(x) = f_{t}(x) - \eta_t {1 \over b}\sum_{i= b(t-1)+1}^{bt} (f_t(x_{j_i}) - y_{j_i})(\Ptx K_{x_{j_i}})(x),
 \eea
 which is exactly \eref{eq:alg1}.

In the next subsections \ref{subsec:pre}-\ref{subsec:der}, we will give the proof of Theorem \ref{thm:simp}. The proof is quite lengthy.
The key is an error decomposition similar as that for classic SGM in \cite{lin2016optimal}, and the basic tools are some concentration inequalities,
operator inequalities and estimates which have already been broadly used in the literature, e.g., \cite{yao2007early,smale2007learning,caponnetto2007optimal,bauer2007regularization,ying2008online,tarres2014online,rudi2015less,rosasco2015learning,lin2016optimal}.

\subsection{Preliminarily Inequalities}\label{subsec:pre}
In this subsection, we introduce some concentration inequalities, operator inequalities and basic estimates that are necessary to the proof of Theorem \ref{thm:simp}.  Proofs for some of these inequalities can be found in Appendix.

{\bf Concentration inequalities:}
\begin{lemma}\label{lem:sampleErr}
Let $0<\lambda $ and $0<\delta<1$.
  Under Assumptions \ref{as:noiseExp} and \ref{as:eigenvalues}, with probability at least $1-\delta$, there holds
 \be\label{eq:sampErr}
\left\|(\TK+\lambda)^{-{1\over 2}}\left(\TX \FH - \SX^*{\bf y}\right) \right\|_{\HK} \leq 4 \left(\|\FH\|_{\infty} + \sqrt{M}\right) \left( { \kappa   \over n \sqrt{\lambda}} + \sqrt{ \sqrt{v} c_{\gamma} \over n\lambda^{\gamma} } \right) \log{2 \over \delta}.
\ee
\end{lemma}

\begin{lemma}
  \label{lem:proErr}
 Let $0<\delta <1$ and $0<\lambda \leq \|\TK\|$. With probability at least $1-\delta,$ the following holds:
  \bea
  \left\| (\TK+\lambda)^{-1/2}(\TK - \TXt)(\TK+\lambda)^{-1/2}  \right\| \leq {4\beta \mcN_{\infty}(\lambda) \over 3 m} + \sqrt{2\beta \mcN_{\infty}(\lambda) \over m}, \quad \beta = \log {8 \kappa^2 \over \lambda\delta}.
  \eea
 Moreover, if $m \geq 8 \mcN_{\infty}(\lambda) \log {8 \kappa^2 \over \lambda \delta},$ then with probability at least $1-\delta,$
 \be\label{eq:concenTXt}
 \left\| (\TK+\lambda)^{-1/2}(\TK - \TXt)(\TK+\lambda)^{-1/2}  \right\| \leq 2/3.
 \ee
\end{lemma}
\begin{rem}\label{rem:neuman}
  The above result also holds when replacing $\TXt$  with $\TX$. Particularly, (since $\mcN_{\infty}(\lambda) \leq \kappa^2 \lambda^{-1}$ implied by \eref{eq:HK})
   if $n \geq 8 \kappa^2 \lambda^{-1} \log {8 \kappa^2 \over \lambda\delta},$
  with probability at least $1-\delta,$
 \be\label{eq:concenTx}
 \left\| (\TK+\lambda)^{-1/2}(\TK - \TX)(\TK+\lambda)^{-1/2}  \right\| \leq 2/3.
 \ee
\end{rem}

\begin{lemma}\label{lem:sumY}
  Under Assumption \ref{as:noiseExp}, with probability at least $1-\delta$ ($\delta\in ]0,1/2[$), there holds
  \bea
  \mcE_{\bf z}(0) - Mv \leq   2Mv\left( {1 \over n} + {2 \over \sqrt{n} } \right) \log {2 \over \delta}.
  \eea
  Particularly, if $n \geq 32 \log^2 {2 \over \delta},$ then
  \be\label{eq:sumY}
  \mcE_{\bf z}(0) \leq M v.
  \ee
\end{lemma}
{\bf Operator Inequalities:}
\begin{lemma}[Cordes Inequality \cite{furuta1989norm}]
\label{lem:cordeine}
  Let $A$ and $B$ be two positive bounded linear operators on a separable Hilbert space. Then
  \bea
  \|A^s B^s\| \leq \|AB\|^s, \quad\mbox{when } 0\leq s\leq 1.
  \eea
\end{lemma}

\begin{lemma}
  Let $A$ and $B$ be two positive bounded linear operators on a separable Hilbert space with $\max(\|A\|,\|B\|) \leq \kappa$ for some non-negative $\kappa.$
  Then for any $0<\zeta\leq 1,$
  \be
\|A^\zeta - B^\zeta\| \leq  \|A -B\|^{\zeta}.
\ee
\end{lemma}
\begin{proof}
  Following from \cite[Theorem 1 and Example 1]{mathe2002moduli}, one can prove the desired result.
\end{proof}

\begin{lemma}\label{lemma:neumanns}
  Let $A$ and $B$ be strictly positive operators on a separable Hilbert space $\mcH$. If
  \bea
  \|A^{-1/2} (B-A) A^{-1/2}\| \leq c <1,
  \eea
  then
  \bea
   \|A^{1/2} B^{-1} A^{1/2}\| \leq {1 \over 1 -c}.
  \eea
\end{lemma}

\begin{lemma}\label{lemma:i-p}(\cite[Proposition 3]{rudi2015less})
  Let $\mcH, \mcK$ be two separable Hilbert spaces, $S: \mcK \to \mcH$ a bounded linear operator and $P:\mcH \to \mcH$ a projection operator with its range as $\overline{\mathrm{range}(S)} $. Then for any bounded linear operator $L : \mcH \to \mcH$ and any $\lambda > 0$ we have
  \bea
  \|(I-P)L\|^2 \leq \lambda \|(SS^*+ \lambda I)^{-1/2} L L^* (SS^*+ \lambda I)^{-1/2}\|.
  \eea
\end{lemma}

\begin{lemma}\label{lem:paroper}(\cite[Proposition 6]{rudi2015less})
  Let $\mcH,\mcK$ be two separable Hilbert spaces, let $A : \mcH \to \mcH$ be a positive
linear operator, $V: \mcK \to \mcH $ a partial isometry and $B: \mcK \to \mcK$ a bounded operator. Then
for all $0\leq r, s\leq 1/2, $ $\|A^r V B V^* A^s \| \leq \|(V^* A V)^r B (V^* A V)^s \|.$
\end{lemma}

\begin{lemma}
  \label{lem:initialerror}
  Let $L$ be a compact, positive operator on a separable Hilbert space $H$. Assume that $\eta_1 \|L\| \leq 1$. Then for $t\in \mN$ and any non-negative integer $k  \leq t - 1,$
and $0< \zeta \leq 1,$
  \be\label{initialerror_interm}
  \| \Pi_{k+1}^t(L) (L+\lambda)^{\zeta}\| \leq \left( \zeta \over \mathrm{e} \sum_{i=k+1}^t \eta_i \right)^{\zeta} + \lambda^{\zeta}.
  \ee
\end{lemma}
{\bf Basic Estimates:}
\begin{lemma}\label{lem:estimate1}
  Let $\theta\in [0,1[$, and $t\in\mN$. Then
 $$ {t^{1-\theta} \over 2} \leq \sum_{k=1}^{t} k^{-\theta} \leq
   {t^{1 - \theta} \over 1-\theta}.
 $$
\end{lemma}

\begin{lemma}
  \label{lem:estimate1a}
  Let $\theta \in \mR$ and $t \in \mN $.
  Then $$ \sum_{k=1}^{t} k^{-\theta} \leq t^{\max(1-\theta,0)} (1+\log t).
 $$
\end{lemma}

\begin{lemma}\label{lem:estimate2}
Let  $q \in \mR$ and $t\in\mN$. Then
  \bea  \sum_{k=1}^{t-1} {1 \over t-k} k^{-q}
 \leq 2 t^{-\min(q,1)} (1+\log t).
  \eea
\end{lemma}

\subsection{Error Decomposition}\label{subsec:errDe}
The key to the proof of Theorem \ref{thm:simp} is an error decomposition.
To state this error decomposition, we need to introduce two auxiliary sequences.
 We first introduce the \emph{projected iteration} (associated with $\mbtx$), defined by $h_1=0$ and
\be\label{eq:alg3}
h_{t+1}=h_t - \eta_t \Ptx(\TX h_t - \TX \FH)  , \qquad t=1, \ldots, T. \ee
The above iterated procedure can not be implemented in practice, as $\FH$ is unknown. Replacing $\FH(x_i)$ by $y_i$, we derive the (projection) {\it sample iteration}, i.e., $g_1 = 0$ and
\be\label{eq:alg2}
g_{t+1}=g_t - \eta_t \Ptx (\TX g_t - \SX^* {\bf y}), \qquad t=1, \ldots, T. \ee
Clearly, $g_t$ is  a $\HK$-valued random variable depending on $\bf z$.

Now we can state our error decomposition as follows.
\begin{pro}
  We have
  \be\label{eq:errorDecompos}
  \mE_{\J}[\mcE(f_t)] - \mcE(\FH) \leq \left(\|h_t - \FH\|_{\rho}  + \|h_t - g_{t}\|_{\rho}\right)^2 + \mE_{\J}\|g_t - f_{t}\|_{\rho}^2.
  \ee
\end{pro}
\begin{proof}
Note that for any $f\in \HK,$ we have (e.g. \cite{bauer2007regularization}),
\bea
\mcE(f) - \inf_{\HK} \mcE = \| f - \FH\|_{\rho}^2.
\eea
Thus,
\bea
\mE_{\J}[\mcE(f_t)] - \mcE(\FH) =  \mE_{\J}[\| f -  \FH\|_{\rho}^2].
\eea
Using an inducted argument, given the sample $\bf z$, it is easy to prove that
\be\label{eq:unbias_index}
\mE_{\J}[f_t] = g_t.
\ee Indeed, taking the expectation with respect to $\J_t$ on both sides of \eref{eq:alg1}, and noting that
$f_{t}$ is depending only on $\J_1,\cdots,\J_{t-1}$ (given any $\mbtx$ and $\bf z$), one has
\bea
\mE_{\J_t}[f_{t+1}] = f_t - \eta_t {1 \over n}\sum_{i=1}^n( f_t (x_{i}) - y_{i}) \Ptx (K_{x_{i}}), \eea
and thus,
\bea
\mE_{\J}[f_{t+1}] = \mE_{\J}[f_t] - \eta_t {1 \over n}\sum_{i=1}^n(\mE_{\J}[f_t] (x_{i}) - y_{i}) \Ptx (K_{x_{i}}) , \qquad t=1, \ldots, T, \eea
which satisfies the iterative relationship given in \eref{eq:alg2}.
 Note that
 \bea
 &&\mE_{\J}[\|f_{t} - \FH\|_{\rho}^2] = \mE_{\J}[\|f_{t} - g_t + g_t - \FH\|_{\rho}^2] \\
 && = \mE_{\J}[\|f_t - g_t\|_{\rho}^2 + \|g_t - \FH\|_{\rho}^2] + 2 \mE_{\J}\la f_t - g_t, g_t - \FH \ra.
 \eea
 Using \eref{eq:unbias_index} to the above equality, and noting that $g_t - \FH$ is depending only on $\mbtx$ and $\bf z$,
 we get
 \bea
 \mE_{\J}[\|f_{t} - \FH\|_{\rho}^2]
 = \mE_{\J}[\|f_t - g_t\|_{\rho}^2 + \|g_t - \FH\|_{\rho}^2].
 \eea
 Thus, the proof can be finished by applying
 \bea
 \|g_t - \FH\|_{\rho} = \|g_t - h_t + h_t -  \FH\|_{\rho}
 \leq \|g_t - h_t\|_{\rho} + \|h_t -  \FH\|_{\rho}.
 \eea
\end{proof}
The above error decomposition is similar as that for the standard SGM in \cite{lin2016optimal}.
The error decomposition \eref{eq:errorDecompos} is composed of three terms.
We refer to the term  $\|h_t - \FH\|_{\rho}$ as the {\it projected bias},
 the term $\|g_t - h_t\|_{\rho}$ as the \emph{sample variance}, and $\|f_t - g_t\|_{\rho}$ as the \emph{computational variance} in this paper.
In the next three subsections, we will focus on estimating these three terms.

\subsection{Projected Bias}\label{subsec:bias}
This subsection is devoted to the estimation of the projected bias, i.e., $\|h_t - \FH\|_{\rho}$.
\begin{pro}\label{pro:bias1}
Let $\eta_t \kappa^2 \leq 1$ for all $t\in \mN$ and $\{h_t\}_t$ be given by \eref{eq:alg3}. Under Assumption \ref{as:regularity}, if $\zeta \leq 1/2,$ then for all $t\in \mN,$
  \begin{multline*}
\|h_{t+1} - \FH \|_{\rho}
\leq \|\TKL^{1/2} \TXL^{-1/2}\|^{2\zeta+1} \left(\left( {1 \over \mathrm{e}\sum_{i=1}^{t} \eta_i} \right)^{\zeta+1/2} + \lambda^{\zeta+1/2}\right) R \\
+ \left(\|\TKL^{1/2}\TXL^{-1/2}\|\|\TXL^{1/2} \TKL^{-1/2}\| \left(1+\lambda\sum_{k=1}^t \eta_k\right)  +1\right) \|(I-\Ptx)\TKL^{1/2}\|^{2\zeta+1} R.
\end{multline*}
\end{pro}
The upper bound from the above estimate is composed of two terms. The first term is referred as the {\it bias} and it is dominated by $\left( {\sum_{i=1}^{t} \eta_i} \right)^{-\zeta-1/2}$ for a suitable choice of $\lambda$ as will be seen in the following subsections.
The second term is referred as the {\it projected variance} and it is dominated by $\|(I-\Ptx)\TKL^{1/2}\|^{2\zeta+1}$.
\begin{proof}
Since $h_{t}$ is given by \eref{eq:alg3}, we know that $h_t \in \HKx$ and thus $\Ptx h_{t} = h_{t}$.
Thus, subtracting both sides of \eref{eq:alg3} by $\Ptx \FH$, with $\Ptx^2=\Ptx = VV^*$,
\begin{gather*}
h_{t+1} - \Ptx \FH
= \Ptx h_t - \Ptx^2 \FH - \eta_t \Ptx(\TX \Ptx h_t - \TX \Ptx^2 \FH) + \eta_t \Ptx \TX (I - \Ptx) \FH \\
= (\Ptx - \eta_t \Ptx \TX \Ptx) (h_t - \Ptx \FH) + \eta_t \Ptx \TX (I - \Ptx) \FH\\
 = V(I - \eta_t V^* \TX V)V^* (h_{t} - \Ptx \FH) + \eta_t VV^* \TX (I - \Ptx) \FH.
\end{gather*}
Using this relationship iteratively, introducing with $V^* V = I$, $f_1=0$ and $V^* \Ptx = V^*$, we get
\bea
h_{t+1} - \Ptx \FH = - V\Pi_1^t(V^* \TX V)V^* \FH +  \sum_{k=1}^t \eta_k V \Pi_{k+1}^t(V^* \TX V) V^*\TX (I - \Ptx) \FH.
\eea
Combining with $\FH = \Ptx \FH + (I - \Ptx) \FH,$
\be\label{eq:r_t}
h_{t+1} - \FH = - V\Pi_1^t(V^* \TX V)V^* \FH + \sum_{k=1}^t \eta_k V \Pi_{k+1}^t(V^* \TX V) V^*\TX (I - \Ptx) \FH - (I - \Ptx) \FH.
\ee
Applying \eref{isometry}, which implies
\bea
\|h_{t+1} - \FH \|_{\rho} = \left\|\sqrt{\TK} (h_{t+1} - \FH) \right\|_{\HK} \leq \left\|\sqrt{\TKL}(h_{t+1} - \FH) \right\|_{\HK},
\eea
we have
\begin{multline*}
\|h_{t+1} - \FH \|_{\rho} \leq \underbrace{\|\TKL^{1/2} V\Pi_1^t(V^* \TX V)V^* \FH \|}_{ \mbox{\bf Term.1}} \\
+ \underbrace{\left\|\TKL^{1/2}\sum_{k=1}^t \eta_k V \Pi_{k+1}^t(V^* \TX V) V^*\TX (I - \Ptx) \FH \right\|}_{ \mbox{\bf Term.2}} + \underbrace{\|\TKL^{1/2}(I - \Ptx) \FH\|}_{ \mbox{\bf Term.3}}.
\end{multline*}
In what follows, we will estimate the above three terms separatively.
First, applying Assumption \ref{as:regularity}, a simple calculation shows that
the first term can be estimated as follows:
\bea
{ \mbox{\bf Term.1}} \leq \|\TKL^{1/2}\TXL^{-1/2}\| \| \TXL^{1/2} V\Pi_1^t(V^* \TX V)V^*\TXL^{\zeta}\| \|\TXL^{-\zeta}\TKL^{\zeta}\| \|\TKL^{-\zeta}\TK^{\zeta} \| R.
\eea
Since $\zeta \leq 1/2,$ we can apply Lemma \ref{lem:cordeine} to get
$$\|\TXL^{-\zeta}\TKL^{\zeta}\|= \|\TXL^{{-1 \over 2} \times 2\zeta}\TKL^{{1 \over 2} \times 2\zeta}\| \leq \|\TXL^{-1/2}\TKL^{1/2}\|^{2\zeta}.$$ Besides,
by Lemmas \ref{lem:paroper} and \ref{lem:initialerror},
\bea
&&\| \TXL^{1/2} V\Pi_1^t(V^* \TX V)V^*\TXL^{\zeta}\|  \\
&\leq& \| (V^*\TXL V)^{1/2} \Pi_1^t(V^* \TX V) (V^*\TXL V)^{\zeta}\|  \\
&=& \| \Pi_1^t(V^* \TX V) (V^*\TXL V)^{\zeta+1/2}\| \\
&\leq&  \left( {\zeta+1/2 \over \mathrm{e}\sum_{i=1}^{t} \eta_i} \right)^{\zeta+1/2} + \lambda^{\zeta+1/2}.
\eea
Thus, the first term can be bounded as
\bea
{ \mbox{\bf Term.1}} \leq \|\TKL^{1/2} \TXL^{-1/2}\|^{2\zeta+1} \left(\left( {\zeta+1/2 \over \mathrm{e}\sum_{i=1}^{t} \eta_i} \right)^{\zeta+1/2} + \lambda^{\zeta+1/2}\right) R.
\eea
For the second term, we have
\begin{align*}
&{ \mbox{\bf Term.2}} \\
\leq& \|\TKL^{1/2}\TXL^{-1/2}\| \|\TXL^{1/2}\sum_{k=1}^t \eta_k V \Pi_{k+1}^t(V^* \TX V) V^* \TXL^{1/2}\| \|\TXL^{1/2} \TKL^{-1/2}\| \|\TKL^{1/2}(I - \Ptx) \FH \|.
\end{align*}
Let $G_t = \sum_{k=1}^t \eta_k \Pi_{k+1}^t (V^* \TX V)$,  for notational simplicity.
Note that
 \begin{align}
 \|\TXL^{1/2} V  G_t V^* \TXL^{1/2}\|
  =& \|(\TXL^{1/2} V  G_t^{1/2}) (G_t^{1/2} V^* \TXL^{1/2})\| \nonumber\\
  =&  \|(\TXL^{1/2} V  G_t^{1/2}) (\TXL^{1/2} V  G_t^{1/2})^*\| \nonumber \\
 =&  \| (\TXL^{1/2} V  G_t^{1/2})^*(\TXL^{1/2} V  G_t^{1/2})\| \nonumber \\
=& \| G_t^{1/2} V^* \TXL^{1/2} \TXL^{1/2} V  G_t^{1/2} \| \nonumber\\
= & \| G_t^{1/2} V^* \TXL V  G_t^{1/2} \| \nonumber \\
\leq&  \| G_t^{1/2} V^* \TX V  G_t^{1/2} \| + \lambda\|G_t^{1/2} V^* V G_t^{1/2}\| \nonumber\\
=& \| G_t^{1/2} V^* \TX V  G_t^{1/2} \| + \lambda\|G_t\| \nonumber \\
=& \| V^* \TX V  G_t \| + \lambda\|G_t\| \nonumber.
\end{align}
Since for all $ x\in[0,\kappa^2]$, $\sum_{k=1}^t \eta_k\Pi_{k+1}^t (x) \leq \sum_{k=1}^t \eta_k$ and
\begin{align*}
\sum_{k=1}^t \eta_k x \Pi_{k+1}^t (x)  =& \sum_{k=1}^t (1 - (1-\eta_k x)) \Pi_{k+1}^t (x)\\
=& \sum_{k=1}^t \left(\Pi_{k+1}^t (x) - \Pi_{k}^t(x) \right) = 1 - \Pi_{1}^t (x) \leq 1.
\end{align*}
Thus, combining with $\|V^* \TX V\| \leq \|\TX\| \leq \kappa^2$ (implied by \eref{eq:HK}), we have
$$\|G_t\| \leq \sum_{k=1}^t \eta_k \|\Pi_{k+1}^t(V^* \TX V)\| \leq \sum_{k=1}^t \eta_k,$$
and
\bea
\| V^* \TX V  G_t \| \leq \sup_{0\leq x\leq \kappa^2} \sum_{k=1}^t \eta_k x \Pi_{k+1}^t (x) \leq 1.
\eea
Therefore,
\be
\|\TXL^{1/2} V  G_t V^* \TXL^{1/2}\| \leq 1 + \lambda \sum_{k=1}^t \eta_k,\label{eq:gtnorm}
\ee
and consequently, we know that the second term can be estimated as
\bea
{ \mbox{\bf Term.2}} \leq \|\TKL^{1/2}\TXL^{-1/2}\|\|\TXL^{1/2} \TKL^{-1/2}\| \left(1+\lambda\sum_{k=1}^t \eta_k\right)  \times { \mbox{\bf Term.3}}.
\eea
What remains is to bound the third term.
Writing $(I-\Ptx)=(I-\Ptx)^{1+2\zeta}$, and applying Assumption \ref{as:regularity}, we have
\begin{align*}
{ \mbox{\bf Term.3}} \leq& \|\TKL^{1/2}(I-\Ptx)\| \|(I-\Ptx)^{2\zeta}\TKL^{(1/2) \times 2\zeta}\| \|\TKL^{-\zeta}\TK^{\zeta}\| R \\
\leq& \|\TKL^{1/2}(I-\Ptx)\| \|(I-\Ptx)\TKL^{1/2}\|^{2\zeta} R,
\end{align*}
where for the last step, we used Lemma \ref{lem:cordeine} with $s=2\zeta\leq 1.$
From the above analysis, we can conclude the proof by noting that $\zeta+1/2 \leq 1$.
\end{proof}
A common choice for the step-size in the literature is $\eta_t = \eta_1 t^{-\theta}$ with $\theta\in [0,1]$.
In this case, by setting $\eta_t = \eta_1 t^{-\theta}$ in Proposition \ref{pro:bias1}, and combining with Lemma \ref{lem:estimate1},
we get the following explicit bound for the projected bias.
\begin{pro}\label{pro:bias2}
 Let $\{h_t\}_t$ be given by \eref{eq:alg3} and $\eta_t = \eta_1 t^{-\theta}$ for all $t\in \mN$ with $0<\eta_1 \leq \kappa^{-2}$ and $\theta\in[0,1[$. Under Assumption \ref{as:regularity}, if $\zeta \leq 1/2,$ then for all $t\in \mN,$
  \begin{multline}\label{eq:bias}
\|h_{t+1} - \FH \|_{\rho}
\leq R \left( \left( {1 \over \eta_1 t^{1-\theta}} \right)^{\zeta+1/2} + \lambda^{\zeta+1/2}\right) \left\|\TXL^{-1/2} \TKL^{1/2}\right\|^{2\zeta+1} \\
+ \left(\|\TKL^{1/2}\TXL^{-1/2}\|\|\TXL^{1/2} \TKL^{-1/2}\| \left(1+{\lambda \eta_1 t^{1-\theta}\over 1-\theta}\right)  +1\right) \|(I-\Ptx)\TKL^{1/2}\|^{2\zeta+1} R.
\end{multline}
\end{pro}
As will be seen in Subsection \ref{subsec:der},
the term $\left\|\TXL^{-1/2} \TKL^{1/2}\right\|^2$ can be well bounded for $\lambda$ that satisfies \eref{eq:concenTx}, using Lemma \ref{lemma:neumanns}.
Moreover, by Lemma \ref{lemma:i-p}, the term $ \|(I-\Ptx)\TKL^{1/2}\|^2$ can be upper bounded in terms of
$\lambda \left\|\mcT_{\tilde{\bf x}\lambda}^{-1/2} \TKL^{1/2}\right\|^2$.

\subsection{Sample Variance}\label{subsec:samVar}
This subsection is devoted to the estimate on the sample variance, i.e., $\|h_{t+1} - g_{t+1}\|_{\rho}.$
\begin{pro}\label{pro:samVar}
Let $h_{t+1}$ and $g_{t+1}$ be given by \eref{eq:alg3} and  \eref{eq:alg2}, respectively, and $\eta_t \kappa^2 \leq 1$ for all $t\in \mN$.
Then for any $\lambda>0$ and $t \in \mN,$ we have
\bea
\|h_{t+1} - g_{t+1}\|_{\rho} \leq \left(1+\lambda\sum_{k=1}^t \eta_k\right) \|\TKL^{1/2} \TXL^{-1/2}\|^2 \|\TKL^{-1/2} ( \TX \FH - \SX^* \mby  ) \|_{\HK}.
\eea
\end{pro}
\begin{proof}
Since $h_{t+1}$ and $g_{t+1}$ are given by \eref{eq:alg3} and  \eref{eq:alg2}, respectively,
\bea
h_{t+1} - g_{t+1} = (I - \eta_t \Ptx \TX)(h_t - g_t) + \eta_t \Ptx \left[ \TX \FH - \SX^* \mby  \right].
\eea
According to the definitions of $h_t$ and $g_t,$ we know that
both $h_t$ and $g_t$ belong to $\HKx$. Recall that $\Ptx $ is the orthogonal  projection operator on $\HKx$ and  $\Ptx = VV^*$.
We thus have $(h_t - g_t)  = \Ptx (h_t - g_t) $. Therefore,
\bea
 h_{t+1} - g_{t+1}
= (I - \eta_t \Ptx \TX) \Ptx (h_t - g_t) + \eta_t \Ptx \left[ \TX \FH - \SX^* \mby  \right].
\eea
Introducing with $\Ptx = VV^*$,
\bea
 h_{t+1} - g_{t+1}
 &=& (I - \eta_t VV^* \TX) VV^* (h_t - g_t) + \eta_t VV^* \left[ \TX \FH - \SX^* \mby  \right]\\
&=& V(I - \eta_t V^* \TX V) V^* (h_t - g_t) + \eta_t VV^* \left[ \TX \FH - \SX^* \mby  \right].
\eea
Applying this relationship iteratively, and noting that $V^*V = I,h_1=g_1=0,$ we get
\bea
h_{t+1} - g_{t+1} = \sum_{k=1}^t \eta_k V\Pi_{k+1}^t (V^* \TX V) V^* \left[ \TX \FH - \SX^* \mby  \right].
\eea
Combining with \eref{isometry}, we get
\bea
&&\|h_{t+1} - g_{t+1}\|  \\
&=&\left\|\sqrt{\TK}\sum_{k=1}^t \eta_k V\Pi_{k+1}^t (V^* \TX V) V^* \left[ \TX \FH - \SX^* \mby  \right] \right\|_{\HK}\\
&\leq& \|\TK^{1/2} \TKL^{-1/2} \| \|\TKL^{1/2} \TXL^{-1/2}\| \left\|\TXL^{1/2} \sum_{k=1}^t \eta_k V\Pi_{k+1}^t (V^* \TX V) V^* \TXL^{1/2}\right\| \\
&& \times \|\TXL^{-1/2}\TKL^{1/2}\| \|\TKL^{-1/2} ( \TX \FH - \SX^* \mby  ) \|_{\HK} \\
&\leq&  \left\|\TXL^{1/2} \sum_{k=1}^t \eta_k V\Pi_{k+1}^t (V^* \TX V) V^* \TXL^{1/2}\right\| \|\TKL^{1/2} \TXL^{-1/2}\|^2 \|\TKL^{-1/2} ( \TX \FH - \SX^* \mby  ) \|_{\HK}\\
&\leq& \left( 1+ \lambda \sum_{k=1}^t \eta_k \right) \|\TKL^{1/2} \TXL^{-1/2}\|^2 \|\TKL^{-1/2} ( \TX \FH - \SX^* \mby  ) \|_{\HK},
\eea
where for the last inequality, we used \eref{eq:gtnorm}.
From the above analysis, we can conclude the desired result. The proof is complete.
\end{proof}

Setting $\eta_t = \eta_1 t^{-\theta}$ with $\theta \in [0,1)$ and $\eta_1>0$ for all $t\in \mN$ in Proposition \ref{pro:samVar}, and then combining with Lemma \ref{lem:estimate1}, we have the following bounds for the sample variance.
\begin{pro}\label{pro:samVar2}
  Assume that $h_{t+1}$ and $g_{t+1}$ are given by \eref{eq:alg3} and  \eref{eq:alg2}, respectively. Let $\eta_t = \eta_1 t^{-\theta}$ with $\theta \in [0,1)$ and $\eta_1\leq \kappa^{-2}$ for all $t\in \mN.$ Then for any $\lambda>0$ and $t \in \mN,$ we have
\be\label{eq:samVar2}
\|h_{t+1} - g_{t+1}\|_{\rho} \leq {1 \over 1-\theta}\left(1+ {\lambda \eta_1 t^{1-\theta}} \right) \|\TKL^{1/2} \TXL^{-1/2}\|^2 \|\TKL^{-1/2} ( \TX \FH - \SX^* \mby  ) \|_{\HK}.
\ee
\end{pro}
Again, the term $\left\|\TXL^{-1/2} \TKL^{1/2}\right\|$ can be well bounded for $\lambda$ that satisfies \eref{eq:concenTx}. The term $\|\TKL^{-1/2} ( \TX \FH - \SX^* \mby  ) \|_{\HK}$ will be estimated by Lemma \ref{lem:sampleErr}.

\subsection{Computational Variance}\label{subsec:comp}
This subsection is devoted to bounding the computational variance, i.e., $\|f_{t+1} - g_{t+1}\|_{\rho}.$
\begin{pro}\label{pro:compVar}
Let $\eta_1 \kappa^2 \leq 1/2$ and for all $t \in [T]$ with $t \geq 2,$
  \be\label{empriskBCon}
  {1 \over \eta_t}  \sum_{k=1}^{t-1}{1 \over k(k+1)}  \sum_{i=t-k}^{t-1} {\eta_i^2} \leq {1 \over 4\kappa^2}.
  \ee
Let $\{g_t\}_{t}$ be given by \eref{eq:alg2}.
Then for all $t \in[T],$ and $\lambda>0,$
  \be\label{eq:compVar1}
  \begin{split}
&\mE_{\J}\|f_{t+1} - g_{t+1}\|_{\rho}^2 \\
\leq&  {8\kappa^2 \over b}  \mcE_{\bf z}(0) \sup_{k\in [t]} \left\{ {1 \over \eta_k k} \sum_{l=1}^k \eta_l\right\}   \left(\sum_{k=1}^{t-1} {\eta_k^2\over \mathrm{e}\sum_{i=k+1}^t \eta_i} + \lambda \sum_{k=1}^{t-1} \eta_k^2 + \eta_t^2\right)\|\TKL^{1/2} \TXL^{-1/2}\|^2.
\end{split}\ee
\end{pro}
To prove this result, we need the following lemma to bound the empirical risks.
Its proof is similar as that for \cite[Lemma D.4]{lin2016optimal}, and will be given in the appendix.
\begin{lemma}\label{lemma:empriskB}
 Assume $\eta_1 \kappa^2 \leq 1/2$ and \eref{empriskBCon} for all $t \in [T]$ with $t \geq 2$.
Then for all $t \in[T],$
\be\label{empiricalBConse}
   \sup_{k\in [t]} \mE_{\J}[\mcE_{\bf z}(f_k)]  \leq 8 \mcE_{\bf z}(0) \sup_{k\in [t]} \left\{ {1 \over \eta_k k} \sum_{l=1}^k \eta_l\right\}.
\ee
\end{lemma}
Now, we are ready to prove Proposition \ref{pro:compVar}.
\begin{proof}[Proof of Proposition \ref{pro:compVar}]
  Since $f_{t+1}$ and $g_{t+1}$ are given by \eref{eq:alg1} and \eref{eq:alg2}, respectively,
\begin{align*}
&f_{t+1} - g_{t+1} \\
=& (f_t - g_t) - \eta_t \left[ {1\over b} \sum_{i=b(t-1)+1}^{bt} (f_t(x_{j_i}) - y_{j_i}) \Ptx K_{x_{j_i}} - \Ptx (\TX g_t - \SX^* \mby)\right] \\
=& (I - \eta_t \Ptx \TX) (f_t - g_t) + {\eta_t \over b} \sum_{i=b(t-1)+1}^{bt} \left[ \Ptx (\TX f_t - \SX^* \mby) - (f_t(x_{j_i}) - y_{j_i}) \Ptx K_{x_{j_i}} \right].
\end{align*}
Note that both $f_t$ and $g_t$ belong to $\HKx$, we thus have $f_t - g_t = \Ptx(f_t - g_t).$ With $\Ptx = VV^*$, $\Ptx^2 = \Ptx,$ and the notation
\be\label{eq:Mk}
M_{t,i} =  \Ptx (\TX f_t - \SX^* \mby) - (f_t(x_{j_i}) - y_{j_i}) \Ptx K_{x_{j_i}} , \qquad i = b(t-1)+1, \cdots, bt,
\ee
we have
\bea
f_{t+1} - g_{t+1} = V(I - \eta_t V^* \TX V) V^* (f_t - g_t) + {\eta_t \over b} VV^*  \sum_{i=b(t-1)+1}^{bt} M_{t,i}.
\eea
Using this relationship iteratively, and introducing with $f_1=g_1=0,$ we get
\bea
f_{t+1} - g_{t+1} = {1 \over b} \sum_{k=1}^t  \sum_{i=b(k-1)+1}^{bk} \eta_k V \Pi_{k+1}^t(V^* \TX V) V^* M_{k,i}.
\eea
Here, we used $V^*V = I.$
Thus,
\begin{eqnarray}
&&\mE_{\J}\|f_{t+1} - g_{t+1}\|_{\rho}^2 \nonumber \\
&=& {1 \over b^2} \mE_{\J}\left\| \sum_{k=1}^t  \sum_{i=b(k-1)+1}^{bk} \eta_k V \Pi_{k+1}^t(V^* \TX V) V^* M_{k,i} \right\|_{\rho}^2 \nonumber \\
&=& {1 \over b^2} \sum_{k=1}^t  \sum_{i=b(k-1)+1}^{bk} \eta_k^2  \mE_{\J}\left\| V \Pi_{k+1}^t(V^* \TX V) V^* M_{k,i} \right\|_{\rho}^2 \label{eq:intertt},
\end{eqnarray}
where for the last equality, we use the fact that if $k\neq k',$ or $k=k'$ but $i\neq i'$\footnote{This is possible only when $b \geq 2$.}, then
\bea
\mE_{\J} \la V \Pi^t_{k+1}(V^*\TX V) V^* M_{k,i}, V \Pi^t_{k'+1}(V^*\TX V)V^* M_{k',i'} \ra_{\rho} = 0.
\eea
Indeed, if $k\neq k',$ without loss of generality, we consider the case $k< k'.$ Recalling that $M_{k,i}$ is given by \eref{eq:Mk} and that given any $\bf z$ (and $\tilde{\bf x}$), $f_{k}$ is depending only on $\J_1,\cdots,\J_{k-1},$
we thus have
\bea
&&\mE_{\J} \la V \Pi^t_{k+1}(V^* \TX V)V^* M_{k,i} , V \Pi^t_{k'+1}(V^* \TX V) V^* M_{k',i'} \ra_{\rho} \\
 &&= \mE_{\J_1,\cdots,\J_{k'-1}} \la V\Pi^t_{k+1}(V^*\TX V)V^* M_{k,i} , V \Pi^t_{l+1}(V^*\TX V) V^* \mE_{\J_{k'}}[M_{k',i'}] \ra_{\rho} = 0.
\eea
If $k=k'$ but $i\neq i',$
 without loss of generality, we assume $i<i'.$
By noting that $f_{k}$ is depending only on $\J_1,\cdots,\J_{k-1}$ and $M_{k,i}$ is depending only on $f_{k}$ and ${j_i}$ (given any sample ${\bf z}$),
\bea
&&\mE_{\J} \la V \Pi^t_{k+1}(V^* \TX V) V^* M_{k,i} , V \Pi^t_{k+1}(V^* \TX V) V^* M_{k,i'} \ra_{\rho} \\
 &&= \mE_{\J_1,\cdots,\J_{k-1}} \la V \Pi^t_{k+1}(V^* \TX V) V^* \mE_{j_i}[ M_{k,i}] , V \Pi^t_{l+1}(V^* \TX V) V^* \mE_{j_{i'}}[M_{k,i'}] \ra_{\rho} = 0.
\eea
Using the isometry property \eref{noiseExp},
\bea
&&\mE_{\J}\left\| V \Pi_{k+1}^t(V^* \TX V) V^* M_{k,i} \right\|_{\rho}^2 \\
&=& \mE_{\J}\left\|\sqrt{\TK} V \Pi_{k+1}^t(V^* \TX V) V^* M_{k,i} \right\|_{\HK}^2 \\
&\leq& \|\sqrt{\TK}\TKL^{-1/2}\|^2 \|\TKL^{1/2} \TXL^{-1/2}\|^2 \|\TXL^{1/2} V \Pi_{k+1}^t(V^* \TX V) V^*\|^2 \mE_{\J}\|M_{k,i}\|_{\HK}^2 \\
&\leq&  \|\TKL^{1/2} \TXL^{-1/2}\|^2 \|\TXL^{1/2} V \Pi_{k+1}^t(V^* \TX V) V^*\|^2 \mE_{\J}\|M_{k,i}\|_{\HK}^2.
\eea
Since for $k\leq t-1,$ by using $\|A\|^2 = \|A^*A\|$ for any bounded linear operator $A,$
\bea
\|\TXL^{1/2} V \Pi_{k+1}^t(V^* \TX V) V^*\|^2 &=& \|  V \Pi_{k+1}^t(V^* \TX V) V^* \TXL  V \Pi_{k+1}^t(V^* \TX V) V^*\| \\
&\leq& \|  V\| \| \Pi_{k+1}^t(V^* \TX V) V^* \TXL  V\| \| \Pi_{k+1}^t(V^* \TX V)\| \| V^*\|\\
&\leq& \| \Pi_{k+1}^t(V^* \TX V) V^* \TXL  V \|\\
&=& \| \Pi_{k+1}^t(V^* \TX V) (V^* \TX  V  + \lambda) \|\\
&\leq&  {1 \over \mathrm{e}\sum_{i=k+1}^t \eta_i}  + \lambda,
\eea
where for the last inequality, we used Lemma \ref{lem:initialerror}.
And by Assumption \eref{eq:HK},
\bea
\mE_{\J}\|M_{k,i}\|_{\HK}^2  &\leq& \mE_{\J}\|(f_k(x_{j_i}) - y_{j_i}) \Ptx K_{x_{j_i}} \|_{\HK}^2 \\
 &\leq& \kappa^2 \mE_{\J}(f_k(x_{j_i}) - y_{j_i})^2 \\
 &=&   \kappa^2 \mE_{\J} [\mcE_{\bf z}(f_k)].
\eea
From the above analysis, we get that for $k\leq t-1$,
\bea
&&\mE_{\J}\left\| V \Pi_{k+1}^t(V^* \TX V) V^* M_{k,i} \right\|_{\rho}^2 \\
&\leq&  \kappa^2 \mE_{\J} [\mcE_{\bf z}(f_k)] \|\TKL^{1/2} \TXL^{-1/2}\|^2 \left( {1 \over \mathrm{e}\sum_{i=k+1}^t \eta_i}  + \lambda\right),
\eea
while for $k=t$,
\bea
&&\mE_{\J}\left\| V \Pi_{k+1}^t(V^* \TX V) V^* M_{k,i} \right\|_{\rho}^2 \\&\leq& \|\TKL^{1/2} \TXL^{-1/2}\|^2\mE_{\J}\left\| M_{k,i} \right\|_{\HK}^2
\leq  \kappa^2 \mE_{\J} [\mcE_{\bf z}(f_k)] \|\TKL^{1/2} \TXL^{-1/2}\|^2.
\eea
Combining with \eref{eq:intertt}, we get
\bea
&&\mE_{\J}\|f_{t+1} - g_{t+1}\|_{\rho}^2 \\
&\leq & {\kappa^2 \over b}  \sup_{k\in [t]}\mE_{\J} [\mcE_{\bf z}(f_k)] \|\TKL^{1/2} \TXL^{-1/2}\|^2  \left(\sum_{k=1}^{t-1} {\eta_k^2\over \mathrm{e}\sum_{i=k+1}^t \eta_i} + \lambda \sum_{k=1}^{t-1} \eta_k^2 + \eta_t^2\right).
\eea
Applying Lemma \ref{lemma:empriskB} to the above, we can get the desired result.
The proof is complete.
\end{proof}
Letting $\eta_t = \eta_1 t^{-\theta}$ with $\theta\in[0,1[$ and some suitable  $\eta_1$ for all $t\in \mN$ in Proposition \ref{pro:compVar}, we can prove the following error bounds for the computational variance.
\begin{pro}\label{pro:comVar2}
  Let $\eta_t = \eta_1 t^{-\theta} $ for all $t \in [T],$ with $\theta \in [0,1[$ and
  \be\label{etaRestri}
  0<\eta_1 \leq {t^{\min(\theta, 1-\theta)} \over 8 \kappa^2 (\log t + 1)} , \qquad \forall t\in [T].
  \ee
  Let $\{g_t\}_{t}$ be given by \eref{eq:alg2} and $\lambda>0$. Then for all $t\in [T]$,
  \be\label{eq:compVar2}
\begin{split}&\mE_{\J}\|f_{t+1} - g_{t+1}\|_{\rho}^2 \\
\leq & {16\kappa^2 \mcE_{\bf z}(0)  \over (1-\theta)}   \left(1+ \lambda \eta_1 t^{1-\theta} \right){\eta_1 t^{-\min(\theta,1-\theta)}\over b } \log (3t)\|\TKL^{1/2} \TXL^{-1/2}\|^2.
\end{split}\ee
\end{pro}
Again the term $\left\|\TXL^{-1/2} \TKL^{1/2}\right\|$ can be well bounded for $\lambda$ that satisfies \eref{eq:concenTx}. The term $\mcE_{\bf z}(0)$ can be upper bounded by a constant by applying Lemma \ref{lem:sumY}.

\begin{proof}
  We will use Proposition \ref{pro:compVar} to prove the result.
  We first need to verify the condition \eref{empriskBCon}.
Note that
\bea
\sum_{k=1}^{t-1} {1 \over k(k+1)} \sum_{i=t-k}^{t-1} \eta_i^2 = \sum_{i=1}^{t-1} \eta_i^2 \sum_{k=t-i}^{t-1} {1 \over k(k+1)} = \sum_{i=1}^{t-1} \eta_i^2 \left( {1\over t-i} - {1 \over t} \right) \leq \sum_{i=1}^{t-1} {\eta_i^2 \over t-i}.
\eea
Substituting with $\eta_i = \eta_1 i^{-\theta},$ and by Lemma \ref{lem:estimate2},
\bea
\sum_{k=1}^{t-1} {1 \over k(k+1)} \sum_{i=t-k}^{t-1} \eta_i^2 \leq \eta_1^2 \sum_{i=1}^{t-1} {i^{-2\theta} \over t-i} \leq 2\eta_1^2 t^{-\min(2\theta,1)} (\log t +1).
\eea
Dividing both sides by $\eta_t $ ($= \eta_1 t^{-\theta}$), and then using \eref{etaRestri},
\bea
{1\over \eta_t }\sum_{k=1}^{t-1} {1 \over k(k+1)} \sum_{i=t-k}^{t-1} \eta_i^2 \leq 2\eta_1 t^{-\min(\theta,1-\theta)} (\log t +1) \leq {1 \over 4 \kappa^2}.
\eea
This verifies \eref{empriskBCon}. Note also that by taking $t= 1$ in \eref{etaRestri}, for all $t\in [T]$ ,
\bea
\eta_t \kappa^2 \leq \eta_1 \kappa^2 \leq {1 \over 8 \kappa^2} \leq {1 \over 2}.
\eea
 We thus can apply Proposition \ref{pro:compVar} to derive \eref{eq:compVar1}. What remains is to control the right-hand side of \eref{eq:compVar1}.
Since
\bea
\sum_{k=1}^{t-1}{\eta_{k}^2 \over \sum_{i=k+1}^t \eta_i} = \eta_1 \sum_{k=1}^{t-1}{k^{-2\theta} \over \sum_{i=k+1}^t i^{-\theta}} \leq \eta_1 \sum_{k=1}^{t-1}{k^{-2\theta} \over (t-k)t^{-\theta}},
\eea
combining with Lemma \ref{lem:estimate2},
\bea
\sum_{k=1}^{t-1}{\eta_{k}^2 \over \sum_{i=k+1}^t \eta_i} \leq 2  \eta_1 t^{-\min(\theta,1-\theta)} (\log t +1).
\eea
Also, by Lemma \ref{lem:estimate1},
\bea
{1 \over \eta_k k} \sum_{l=1}^k \eta_l = {1 \over k^{1-\theta}} \sum_{l=1}^k l^{-\theta} \leq {1 \over 1-\theta},
\eea
and by Lemma \ref{lem:estimate1a},
\bea
\sum_{k=1}^{t-1} \eta_{k}^2 = \eta_1^2 \sum_{k=1}^{t-1} k^{-2\theta} \leq \eta_1^2 t^{\max(1-2\theta,0)} (\log t +1).
\eea
Introducing the last three estimates into \eref{eq:compVar1} and using that $\eta_t^2 \kappa^2 \leq \eta_1 t^{-\min(\theta,1-\theta)}$ (implied by \eref{etaRestri}), we get
 \bea
&&\mE_{\J}\|f_{t+1} - g_{t+1}\|_{\rho}^2 \\
&\leq & {8\kappa^2 \mcE_{\bf z}(0)\over b(1-\theta)}   \left(2/\mathrm{e} + \lambda \eta_1 t^{1-\theta} +1\right)\eta_1 t^{-\min(\theta,1-\theta)} (\log t + 1)\|\TKL^{1/2} \TXL^{-1/2}\|^2 \\
&\leq & {16\kappa^2 \mcE_{\bf z}(0) \over b(1-\theta)}   \left(1 + \lambda \eta_1 t^{1-\theta} \right)\eta_1 t^{-\min(\theta,1-\theta)} \log (3t)\|\TKL^{1/2} \TXL^{-1/2}\|^2,
\eea
which leads to the desired result. The proof is complete.
\end{proof}

\subsection{Deriving Total Error Bounds}\label{subsec:der}
This subsection is devoted to deriving total error bounds for NySGM.

Combining Propositions \ref{pro:bias2}, \ref{pro:samVar2} and \ref{pro:comVar2} with the error decomposition \eref{eq:errorDecompos},
and then applying concentration inequalities from Lemmas \ref{lem:sampleErr}, \ref{lem:proErr} and \ref{lem:sumY}
to bound the related terms, we can prove the following total error bounds for NySGM.
\begin{thm}\label{thm:comp}
Under Assumptions \ref{as:noiseExp}, \ref{as:regularity} and \ref{as:eigenvalues}, let $\zeta\leq 1/2$,
\be\label{eq:num}
n \geq \left({16\kappa^2 p \over \|\TK\|} \log {8\kappa^2p \over \|\TK\| \delta} \right)^{p} + 32 \log^2 {2\over \delta}, \qquad p ={1+ {1 \over 2\zeta+\gamma}},
\ee
\bea
\mbox{and} \quad m \geq 8 \mcN_{\infty}(\lambda) \log {8 \kappa^2 \over \lambda \delta}, \quad \lambda = \|\TK\| n^{-{1\over 2\zeta+\gamma+1}}.
\eea
  Set $\eta_t = \eta_1 t^{-\theta} $ for all $t \in [T],$ with $\theta \in [0,1[$ and $\eta_1$ satisfying \eref{etaRestri}. Then with probability at least $1-4\delta$ ($\delta\in (0,1/4)$,) for all $t\in [T],$ we have
  \begin{align}
&\mE_{\J}[\mcE(f_{t+1})] - \mcE(\FH)
\leq q_2 \left(1+ \|\TK\| n^{-{1\over 2\zeta+\gamma+1}}\eta_1 t^{1-\theta} \right)^2 n^{-{2\zeta+1 \over 2\zeta+ \gamma+1}} \log^2{2 \over \delta}  \nonumber\\
&+ q_1 (\eta_1 t^{1-\theta})^{-(2\zeta+1)} + q_3  \left(1+ \|\TK\| n^{-{1\over 2\zeta+\gamma+1}}\eta_1t^{1-\theta} \right) \eta_1 b^{-1}  t^{-\min(\theta,1-\theta)} \log (3 t),\label{eq:bound}
\end{align}
where $q_1,q_2$ and $q_3$ are positive constants depending only on $\kappa,c_{\gamma},\theta,M, v,\zeta$ and $R$, and will be given explicitly in the proof.
\end{thm}

\begin{proof}
First note that the maximum of the function $g(x) = \mathrm{e}^{-cx}x^{\zeta}$( with $c>0$) over $ \mR_+ $ is achieved at $x_{\max}= \zeta/c,$ and thus
  \be\label{exppoly}
  \sup_{x \geq 0} \mathrm{e}^{-cx} x^{\zeta} =  \left({\zeta \over \mathrm{e}c} \right)^{\zeta}.
  \ee
The condition $n \geq 8\kappa^2 \lambda^{-1} \log {8 \kappa^2 \over \lambda \delta}$ in Remark \ref{rem:neuman} is satisfied
when \eref{eq:num} holds. Indeed, with $\lambda= \|\TK\| n^{-{1\over 2\zeta+\gamma+1}},$
\bea
&&{8\kappa^2\over \|\TK\|} \log {8 \kappa^2 \over \lambda \delta} \\
&=& {8\kappa^2\over \|\TK\|(2\zeta+\gamma)} \log \left({8 \kappa^2 n^{1\over 2\zeta+\gamma+1} \over \|\TK\| \delta}\right)^{2\zeta+\gamma} \\
&\leq& {8\kappa^2\over \|\TK\|(2\zeta+\gamma)} \left( \log {1\over \mathrm{e}c} +  c \left({8 \kappa^2 n^{1\over 2\zeta+\gamma+1} \over \|\TK\| \delta}\right)^{2\zeta+\gamma} \right),
\eea
where for the last inequality,
we used \eref{exppoly} with $\zeta'=1$ and $x = \left({8 \kappa^2 n^{1\over 2\zeta+\gamma+1} \over \|\TK\| \delta}\right)^{2\zeta+\gamma}$ and $c >0.$
Choosing $c= \left( \|\TK\| \over 8\kappa^2 \right)^{2\zeta+\gamma+1} \delta^{2\zeta + \gamma} {2\zeta+\gamma \over 2},$ and by \eref{eq:num}, we get that
\bea
{8\kappa^2\over \|\TK\|} \log {8 \kappa^2 \over \lambda \delta} \leq {8\kappa^2 p\over \|\TK\|}\log {8\kappa^2 p \over \|\TK\|\delta }  +  {1\over 2}  n^{1\over p} \leq n^{1\over p},
\eea
which verifies the condition $n \geq 8\kappa^2 \lambda^{-1} \log {8 \kappa^2 \over \lambda \delta}$ in Remark \ref{rem:neuman}.
Now, we can apply Lemmas \ref{lem:sampleErr}, \ref{lem:proErr}, \ref{lem:sumY} and Remark \ref{rem:neuman}, to get that with probability at least $1-4\delta,$ \eref{eq:sampErr}, \eref{eq:concenTXt}, \eref{eq:concenTx} and \eref{eq:sumY} hold.

We next verify that \eref{eq:concenTx} implies that
\be\label{eq:diffeOperTx1}
\|\TXL^{1/2} \TKL^{-1/2}\|^2  \leq 5/3
\ee
and
\be\label{eq:diffeOperTx2}
\|\TXL^{-1/2} \TKL^{1/2}\|^2  \leq 3.
\ee
Indeed,
\bea
\|\TXL^{1/2} \TKL^{-1/2}\|^2
& =& \| \TKL^{-1/2} \TXL \TKL^{-1/2}\| \\
&\leq& \| \TKL^{-1/2} \TKL \TKL^{-1/2}\|  + \| \TKL^{-1/2} (\TX - \TK) \TKL^{-1/2}\|\\
&\leq & 1 + 2/3 = 5/3,
\eea
and by Lemma \ref{lemma:neumanns},
\bea
\|\TXL^{-1/2} \TKL^{1/2}\|^2
=  \|\TKL^{1/2} \TXL^{-1} \TKL^{1/2}\|
\leq 3.
\eea
Similarly, by Lemma \ref{lemma:neumanns}, we also see that \eref{eq:concenTXt} implies
\be\label{eq:diffeOperTTx2}
\| \TKL^{1/2} (\TTX  + \lambda I)^{-1} \TKL^{1/2} \| \leq 3.
\ee
 Note that $\Ptx$ is an orthogonal projection operator on the range of $\STX^*$.
 Thus, applying Lemma \ref{lemma:i-p},
 \bea
  \|(I- \Ptx)\TKL^{1/2}\|^2 &\leq& \lambda \|(\STX^*\STX + \lambda I)^{-1/2}\TKL(\STX^*\STX + \lambda I)^{-1/2}\| \\
  &=& \lambda\|\TKL^{1/2}(\TTX  + \lambda I)^{-1}\TKL^{1/2}\|.
 \eea
 Combining with \eref{eq:diffeOperTTx2},  we get
\be\label{eq:projErr}
  \|(I- \Ptx)\TKL^{1/2}\|^2\leq 3 \lambda.
 \ee

Introducing \eref{eq:bias}, \eref{eq:samVar2} and \eref{eq:compVar2} into the error decomposition \eref{eq:errorDecompos},
and then substituting with \eref{eq:diffeOperTx1}, \eref{eq:diffeOperTx2}, \eref{eq:projErr} and \eref{eq:sumY},
 and by noting that $\zeta \leq 1/2,$ with a simple calculation, we get,
\begin{multline*}
\mE_{\J}[\mcE(f_{t+1})] - \mcE(\FH)
\leq
 {48 M v \kappa^2 \over b(1-\theta)}   \left(1+ \lambda\eta_1 t^{1-\theta} \right) \eta_1 t^{-\min(\theta,1-\theta)} \log (3 t) \\
+ \left(3R (\eta_1 t^{1-\theta})^{-\zeta-{1\over2}} + 3R\lambda^{\zeta+{1\over 2}} + {10R \over 1-\theta} \lambda^{\zeta+1/2}\left(1+ \lambda\eta_1 t^{1-\theta} \right) \right.\\
\left. + {3\over 1-\theta}\left(1 + \lambda \eta_1 t^{1-\theta}\right) \|\TKL^{-1/2}(\TX \FH - \SX^* {\bf y})\|_{\HK} \right)^2.
\end{multline*}
Applying \eref{eq:sampErr} and $\lambda = \|\TK\| n^{-{1\over 2\zeta+\gamma+1}},$ which implies
\bea
&&\|\TKL^{-1/2}(\TX \FH - \SX^* {\bf y})\|_{\HK} \\
&\leq& 4 \left(\|\FH\|_{\infty} + \sqrt{M}\right) \left( \sqrt{ \kappa   \over \|\TK\| } + \sqrt{ \sqrt{v} c_{\gamma} \over \|\TK\|^{\gamma} } \right)  n^{-{\zeta+1/2 \over 2\zeta+ \gamma+1}} \log{2 \over \delta},
\eea
and by a simple inequality $(a+b)^2 \leq 2a^2 + 2b^2$, we get the desired result with
\bea
q_1 = 18R^2, q_3 = {48M v \kappa^2 \over 1-\theta}
\eea
and
\bea
q_2 = 2\left({13R\|\TK\|^{\zeta+1/2}\over 1-\theta }  + {12 \over 1-\theta} \left(\|\FH\|_{\infty} + \sqrt{M}\right) \left( \sqrt{ \kappa   \over \|\TK\| } + \sqrt{ \sqrt{v} c_{\gamma} \over \|\TK\|^{\gamma} } \right) \right)^2.
\eea
The proof is complete.
\end{proof}

With the above theorem, we are now ready to prove Theorem \ref{thm:simp}.
\begin{proof}
  [Proof of Theorem \ref{thm:simp}]
  Following from Theorem \ref{thm:comp}, with $\theta=0$, we get that with probability at least $1-\delta,$ for all $t\in [T]$, there holds
  \bea
\mE_{\J}[\mcE(f_{t+1})] - \mcE(\FH)
&\leq& q_1 (\eta t)^{-(2\zeta+1)} +¡¡q_2 \log^2{8 \over \delta} \left(1+ \|\TK\| n^{-{1\over 2\zeta+\gamma+1}}\eta t \right)^2 n^{-{2\zeta+1 \over 2\zeta+ \gamma+1}}   \\
&&+ q_3  \left(1+ \|\TK\| n^{-{1\over 2\zeta+\gamma+1}}\eta t \right) \eta b^{-1} \log (3 t).
\eea
Introducing with
\bea
&&\left(1+ \|\TK\| n^{-{1\over 2\zeta+\gamma+1}}\eta t \right)^2 n^{-{2\zeta+1\over 2\zeta+\gamma+1}}\\
&\leq& 2n^{-{2\zeta+1 \over 2\zeta+\gamma+1}} + 2\|\TK\|^2 \left(n^{-{1\over 2\zeta+\gamma+1}}\eta t\right)^2n^{-{2\zeta+1 \over 2\zeta+\gamma+1}} \\
&\leq&2 (\eta t)^{-(2\zeta+1)} + (2+ 2\|\TK\|^2) \left(n^{-{1\over 2\zeta+\gamma+1}}\eta t\right)^2n^{-{2\zeta+1 \over 2\zeta+\gamma+1}},
\eea
and by a simple calculation, one can get the desired result.
\end{proof}

\begin{proof}
  [Proof of Theorem \ref{thm:simple}]
  Using Corollary \ref{cor:2} with $\zeta=0$, one can prove the desired result.
\end{proof}
Following the proof of Theorem \ref{thm:comp}, we get the following result for any learning sequence generated by Algorithm \ref{alg:1} with any orthogonal projection operator $P$ that is predefined before the updating procedure, rather than $\Ptx$.
\begin{thm}\label{thm:genOper}
Under Assumptions \ref{as:noiseExp}, \ref{as:regularity} and \ref{as:eigenvalues}, let $\zeta\leq 1/2$ and \eref{eq:num}. Assume that the learning sequence $\{f_t\}_t$ is generated by Algorithm \ref{alg:1}, with any orthogonal projection operator $P$ that is predefined before the updating procedure.
Set $\eta_t = \eta_1 t^{-\theta} $ for all $t \in [T],$ with $\theta \in [0,1[$ and $\eta_1$ satisfying \eref{etaRestri}. Then with probability at least $1-3\delta$ ($\delta\in (0,1/3)$,) for all $t\in [T],$ we have
  \begin{gather*}
\mE_{\J}[\mcE(f_{t+1})] - \mcE(\FH)
\lesssim (\eta_1 t^{1-\theta})^{-(2\zeta+1)} +   \left(1+  n^{-{1\over 2\zeta+\gamma+1}}\eta_1t^{1-\theta} \right) \eta_1 b^{-1}  t^{-\min(\theta,1-\theta)} \log (t)  \\ +¡¡ \left(1+  n^{-{1\over 2\zeta+\gamma+1}}\eta_1 t^{1-\theta} \right)^2  \log^2{2 \over \delta} \left( n^{-{2\zeta+1 \over 2\zeta+ \gamma+1}} + \|(I - P)\TKL\|^{4\zeta+2}  \right),
\end{gather*}
where $\lambda = \|\TK\| n^{-{1\over 2\zeta+\gamma+1}}.$
\end{thm}

Now let the orthogonal projection operator $P = \Ptx$, with $\mbtx$ sampled from the
training set using the approximate leverage scores (ALS) Nystr\"{o}m methods \cite{drineas2012fast,el2014fast,cohen2015uniform}. Combining with \cite[Lemma 7]{rudi2015less}, we get the following results for Algorithm \ref{alg:1} with ALS Nystr\"{o}m.

\begin{corollary}
  Under the same assumptions of Theorem \ref{thm:genOper}, let $P = \Ptx$, with $\mbtx$ sampled from the training set using  ALS Nystr\"{o}m methods. Then with probability at least $1-4\delta,$ ($0<\delta\leq 1/4$), \eref{eq:bound} holds, provided that
  \bea
  m \gtrsim n^{{\gamma \over 2\zeta+\gamma+1}} \log {n}.
  \eea
\end{corollary}
Consequently, one can easily prove Remark \ref{rem:theorem} (II). 

\subsection{SGM with Cross-Validation}\label{subsec:adaptive}
In this subsection, we introduce the cross-validation approach for tuning the step-size, and prove its statistical results.
We assume that $|y| \leq M$ and $|\FH(x)|\leq M$ almost surely for some $M>0.$
 Let $\bar{\bf z}=\{\bar{z}_1,\cdots, \bar{z}_{\bar{m}}\}$ be another sample,
 which we assume to be i.i.d. drawn from $\rho$ and moreover, is independently from ${\bf z}.$ Let $f_{{\bf z},\lambda,\J}(x) = \la \omega_{{\bf z},\lambda}, x \ra $, with $\omega_{{\bf z},\lambda}$ generated by the learning algorithm/procedure (\ref{alg:1}) with the constant step-size $\eta_t=\lambda$ and the index set $\J$.
 Let $\Lambda = \{\lambda_t\}_{t}$ be a finite set with $0<\lambda_t \leq \kappa^2$ for all $t$. Define the truncation operator $T_M:\LR \to \LR$ as
\bea
T_Mf (x) = (|f(x)|\wedge M) \sign f(x),
\eea
and the data-dependent choice parameter as
\bea
\hat{\lambda} = \argmin_{\lambda \in \Lambda} {1 \over \bar{m}}\sum_{i=1}^{\bar {m}} \left( T_M f_{{\bf z},\lambda,\J}(\bar{x_i}) - \bar{y}_i \right)^2.
\eea
The final learning estimator induced with the parameter $\hat{\lambda}$ in this subsection is given by
\bea
f^{tot} = T_M f_{{\bf z},{\hat{\lambda}},\J}.
\eea
Then \cite[Eq.(56)]{caponnetto2010cross} with probability at least $1-\delta,$
\bea
\|f^{tot} - \FH\|_{\rho}^2 \leq 2\|T_M f_{\bf z, \lambda^*,\J} - \FH\|_{\rho}^2 + {80M^2 \over \bar{m}} \log {2|\Lambda| \over \delta},
\eea
where
\bea
\lambda^* = \argmin_{\lambda \in \Lambda} \|T_M f_{{\bf z},{{\lambda}},\J} - \FH\|_{\rho}.
\eea
As a result, if $\eta \in \Lambda$ such that with probability at least $1-\delta,$
\bea
\mE_{\J}\|T_M f_{\bf z, \eta,\J} - \FH\|_{\rho}^2 \leq \epsilon,
\eea
then with probability at least $1-2\delta,$
\bea
\mE_{\J}\|f^{tot} - \FH\|_{\rho}^2 \leq 2\epsilon + {80M^2 \over \bar{m}} \log {2|\Lambda| \over \delta}.
\eea


\bibliography{sigd}
\bibliographystyle{plain}
\appendix
\section{Proofs for Subsection \ref{subsec:pre}}

\begin{lemma}
  \label{lem:Bernstein}
  Let $w_1,\cdots,w_m$ be i.i.d random variables in a Hilbert space with norm $\|\cdot\|$. Suppose that
   there are two positive constants $B$ and $\sigma^2$ such that
   \be\label{bernsteinCondition}
   \mE [\|w_1 - \mE[w_1]\|^l] \leq {1 \over 2} l! B^{l-2} \sigma^2, \quad \forall l \geq 2.
   \ee
   Then for any $0< \delta <1$, the following holds with probability at least $1-\delta$,
  $$ \left\| {1 \over m} \sum_{k=1}^m w_m - \mE[w_1] \right\| \leq 2\left( {B \over m} + {\sigma \over \sqrt{ m }} \right) \log {2 \over \delta} .$$
In particular, \eref{bernsteinCondition} holds if
\be\label{bernsteinConditionB}
\|w_1\| \leq B/2 \ \mbox{ a.s.}, \quad \mbox{and } \mE [\|w_1\|^2] \leq \sigma^2.
\ee
\end{lemma}
\begin{proof}
  We refer to \cite{caponnetto2007optimal} for a proof, which is based on the results in \cite{pinelis1986remarks}.
\end{proof}

\begin{proof}
  [Proof of Lemma \ref{lem:sampleErr}]
  The proof can be also found in \cite{lin2016optimal}.
  We will use Lemma \ref{lem:Bernstein} to prove this result.
  For all $i \in [n],$ let $w_i =   (\FH(x_i) - y_i) \TKL^{-{1\over 2}} K_{x_i}.$
Obviously,
\bea
(\TK + \lambda)^{-{1\over 2}}\left(\TX \FH - \SX^*{\bf y}\right) =  {1 \over n} \sum_{i=1}^n  w_i.
\eea
and from the definitions of $f_{\rho}$ (see \eref{regressionfunc}) and $\IK$, for any $\omega=\omega_i, i \in [n],$
\bea
\mE[w] = \mE_{x} [(\FH(x) - f_{\rho}(x)) \TKL^{-{1\over 2}} K_{x}] = \TKL^{-{1\over 2}} (\TK \FH - \IK^* f_{\rho}) = 0,
\eea
where for the last equality, we used \eref{eq:embeddingEqua}.
We next estimate the constants $B$ and $\sigma^2(w_1)$ in \eref{bernsteinCondition}.
Note that for any $l \geq 2,$ by using H\"{o}lder's inequality twice,
\bea
\int_Y (\FH(x) - y)^l d\rho(y|x) &\leq& \int_Y  2^{l-1} (|\FH(x)|^l + |y|^l) d\rho(y|x) \\
&\leq& 2^{l-1}  \left(\|\FH\|_{\infty}^l + \left(\int_{Y} |y|^{2l} d\rho(y|x)\right)^{1\over 2} \right) \\
&\leq& 2^{l-1} \left(\|\FH\|_{\infty}^l + \sqrt{l! M^{l} v} \right) \\
&\leq& {1 \over 2}  l! \left(2\|\FH\|_{\infty} + 2\sqrt{M}\right)^{l} \sqrt{  v},
\eea
where for the third inequality, we used Assumption \ref{as:noiseExp}, and $\sqrt{l!} \leq l!, a^l + b^l \leq (a+b)^l, \forall a,b \in \mR^+$, $v\geq1$ for the last inequality.
Thus,
\bea
\mE[\|w\|_{\HK}^l]
&=& \int_X \|\TKL^{-{1\over 2}} K_x\|^l \int_Y (\FH(x) - y)^l d\rho(y|x)  d\rho_X(x) \\
&\leq&  {1 \over 2}  l! \left(2\|\FH\|_{\infty} + 2\sqrt{M}\right)^{l} \sqrt{  v} \int_X  \|\TKL^{-{1\over 2}} K_x\|^l d \rho_X(x).
\eea
Using Assumption \eref{eq:HK} which imples
$$\|\TKL^{-{1\over 2}} K_x\|_{\HK} \leq {\|K_x\|_{\HK} \over \sqrt{\lambda}} \leq {\kappa \over \sqrt{\lambda}},$$
we get that
\bea
\mE[\|w\|_{\HK}^l]
 \leq
 {1 \over 2}  l! \left(2\|\FH\|_{\infty} + 2\sqrt{M}\right)^{l} \sqrt{  v} \left({\kappa \over \sqrt{\lambda}}\right)^{l-2}  \int_{X} \|\TKL^{-{1\over 2}} K_x\|_{\HK}^2 d\rho_{X}(x).
\eea
Using the fact that $\|\TKL^{-{1\over 2}} K_x\|_{\HK}^2 = \la \TKL^{-{1\over 2}} K_x, \TKL^{-{1\over 2}} K_x  \ra_{\HK} = \la \TKL^{-1} K_x, K_x  \ra_{\HK}$
we know that
\bea
\int_{X} \|\TKL^{-{1\over 2}} K_x\|_{\HK}^2 d\rho_{X}(x) = \mcN(\lambda) \leq c_{\gamma} \lambda^{-\gamma},
\eea
where for the last inequality, we used Assumption \ref{as:eigenvalues}.
Therefore,
\bea
\mE[\|w\|_{\HK}^l]
 \leq {1 \over 2}  l!  \left({2\kappa(\|\FH\|_{\infty} +\sqrt{M}) \over \sqrt{\lambda}}\right)^{l-2} \left(2\|\FH\|_{\infty} + 2\sqrt{M}\right)^{2} \sqrt{v} c_{\gamma} \lambda^{-\gamma}.
\eea
Applying Berstein inequality from Lemma \ref{lem:Bernstein} with $$B= {2\kappa(\|\FH\|_{\infty} +\sqrt{M}) \over \sqrt{\lambda}} \quad\mbox{and} \quad \sigma = 2\left(\|\FH\|_{\infty} + \sqrt{M}\right) \sqrt{ \sqrt{v} c_{\gamma} \lambda^{-\gamma} },$$ we get the desired result.
\end{proof}

\begin{lemma}
\label{lem:concentrSelfAdjoint}
  Let $\mcX_1, \cdots, \mcX_m$ be a sequence of independently and identically distributed self-adjoint Hilbert-Schmidt operators on a separable Hilbert space.
  Assume that $\mE [\mcX_1] = 0,$ and $\|\mcX_1\| \leq B$ almost surely for some $B>0$. Let $\mathcal{V}$ be a positive trace-class operator such that $\mE[\mcX_1^2] \preccurlyeq \mathcal{V}.$
Then with probability at least $1-\delta,$ ($\delta \in ]0,1[$), there holds
\bea
\left\| {1 \over m} \sum_{i=1}^m \mcX_i \right\| \leq {2B \beta \over 3m} + \sqrt{2\|\mathcal{V}\|\beta \over m }, \qquad \beta = \log {4 \tr \mathcal{V} \over \|\mathcal{V}\|\delta}.
\eea
\end{lemma}
\begin{proof}
 Following from the argument in \cite[Section 4]{minsker2011some}, we can generalize \cite[Theorem 7.3.1]{tropp2012user} from a sequence of self-adjoint matrices to a sequence of self-adjoint Hilbert-Schmidt operators on a separable Hilbert space, and get that for any $t \geq \sqrt{{\|\mcV\|\over m}} + {B\over 3m},$
  \be\label{eq:conf}
  \Pr\left( \left\| {1 \over m} \sum_{i=1}^m \mcX_i \right\| \geq t \right) \leq  {4\tr \mcV \over \|\mcV\|} \exp\left( {- mt^2 \over 2\|\mcV\| + 2Bt/3 }\right).
  \ee
  Rewriting
  \bea
  {4\tr \mcV \over \|\mcV\|} \exp\left( {- mt^2 \over 2\|\mcV\| + 2Bt/3 }\right) = \delta,
  \eea
 as a quadratic equation with respect to the variable $t$, and then solving the quadratic equation, we get
  \bea
  t_0 = {B\beta \over 3m} + \sqrt{\left({B \beta \over 3m}\right)^2 + {2\beta \|\mcV\| \over m}} \leq {2B\beta \over 3m} + \sqrt{{2\beta \|\mcV\| \over m}} := t^*,
  \eea
  where we used $\sqrt{a+b} \leq \sqrt{a} + \sqrt{b},\forall a,b>0.$ Note that $\beta>1$, and thus $t_0\geq \sqrt{{\|\mcV\|\over m}} + {B\over 3m}.$ By
  \bea
  \Pr\left( \left\| {1 \over m} \sum_{i=1}^m \mcX_i \right\| \geq t_* \right) \leq \Pr\left( \left\| {1 \over m} \sum_{i=1}^m \mcX_i \right\| \geq t_0 \right),
  \eea
  and applying \eref{eq:conf} to bound the left-hand side, one can get the desire result.
\end{proof}

\begin{proof}[Proof of Lemma \ref{lem:proErr}]
The proof is essentially given by \cite{rudi2015less}, see also \cite{hsu2012random} for the finite dimensional cases.
  We will use Lemma \ref{lem:concentrSelfAdjoint} to prove the result.
  Let $\mcX_i = \TKL^{-1/2} (\TK - \TK_{\tilde{x}_i}) \TKL^{-1/2} .$ Then $\TKL^{-1/2} (\TK - \TXt) \TKL^{-1/2} = {1\over m}\sum_{i=1}^m \mcX_i.$  Obviously, for any $\mcX= \mcX_i$, $\mE[\mcX]=0,$ and
  \bea
  \|\mcX\| \leq \mE\left[\|\TKL^{-1/2} \TK_{\tilde{x}} \TKL^{-1/2} \|\right] + \|\TKL^{-1/2} \TK_{\tilde{x}} \TKL^{-1/2}\| \leq 2\mcN_{\infty}(\lambda),
  \eea
  where for the last inequality, we used
  \bea
  \|\TKL^{-1/2} \TK_{\tilde{x}} \TKL^{-1/2}\| \leq \tr(\TKL^{-1/2} \TK_{\tilde{x}} \TKL^{-1/2}) = \tr(\TKL^{-1} \TK_{\tilde{x}}) = \la \TKL^{-1} K_{\tilde{x}}, K_{\tilde{x}} \ra_{\HK} \leq \mcN_{\infty}(\lambda).
  \eea
  Also, by $\mE(A - \mE A)^2 \preccurlyeq \mE A^2,$
  \bea
  \mE\mcX^2 &\preccurlyeq& \mE(\TKL^{-1/2} \TK_{\tilde{x}} \TKL^{-1/2})^2 = \mE [\la \TKL^{-1} K_{\tilde{x}}, K_{\tilde{x}} \ra_{\HK} \TKL^{-1/2} K_{\tilde{x}} \otimes  K_{\tilde{x}} \TKL^{-1/2}] \\
  &\preccurlyeq & \mcN_{\infty}(\lambda) \mE [ \TKL^{-1/2} K_{\tilde{x}} \otimes  K_{\tilde{x}} \TKL^{-1/2}] = \mcN_{\infty}(\lambda) \TKL^{-1} \TK,
  \eea
  $\mcN(\TKL^{-1} \TK) \leq \kappa^2 \lambda^{-1}$ implied by \eref{eq:HK},
  and $1 \geq \|\TKL^{-1} \TK\| = {\|\TK\| \over \|\TK\| + \lambda} \geq 1/2, $ since $\lambda\leq \|\TK\|.$
  Now, the first part of this lemma can be proved by applying Lemma \ref{lem:concentrSelfAdjoint}.
The second part follows directly from the first part by a simple calculation.
The proof is complete.
\end{proof}
\begin{proof}[Proof of Lemma \ref{lem:sumY}]
The proof can be also found in \cite{lin2016optimal}.
  Following from \eref{noiseExp},
  \bea
  \int_{Z} y^{2l} d\rho \leq {1 \over 2} l! M^{l-2} \cdot (2M^2 v), \qquad \forall l\in \mN,
  \eea
  and
  \bea
  \int_{Z} y^{2} d\rho \leq   M  v.
  \eea
  Therefore,
  \bea
  \int_{Z} |y^2 - \mE y^2|^l d\rho &\leq& \int_{Z} \max(|y|^{2l}, (\mE y^2)^l) d\rho \\
  &\leq& \int_{Z} (|y|^{2l}+ (\mE y^2)^l) d\rho \\
  &\leq& {1 \over 2} l! M^{l-2} \cdot (2M^2 v) +   (M v)^l \\
  &\leq& {1 \over 2} l! (M v)^{l-2} (2M v)^2,
  \eea
  where for the last inequality we used $v \geq 1.$
  Applying Lemma \ref{lem:Bernstein}, with $\omega_i = y_i^2$ for all $i\in[n]$, $B= M v$ and $\sigma = 2Mv,$ we know that with probability at least $1-\delta_3,$ there holds
  \bea
  {1 \over n}\sum_{i=1}^n y_i^2 - \int_{Z}y^2 d\rho \leq 2Mv\left( {1 \over n} + {2 \over \sqrt{n} } \right) \log {2 \over \delta_3}.
  \eea
The proof is complete.
\end{proof}

\begin{proof}[Proof of Lemma \ref{lemma:neumanns}]
  A simple calculation shows that
  \bea
  A^{1/2} B^{-1} A^{1/2} = (I - A^{-1/2} (A-B) A^{-1/2})^{-1}.
  \eea
  Let $L = A^{-1/2} (A-B) A^{-1/2}.$
  By the condition $\|A^{-1/2} (A-B) A^{-1/2}\| \leq c <1,$ we know that $\|L\|\leq c<1$, thus the Neumann series $\sum_{k=1}^{\infty} L^k$ converges and moreover,
  \bea
  \left\|\sum_{k=1}^{\infty} L^k \right\|  = \|(I - L)^{-1}\| \leq {1 \over 1-c},
  \eea
  which leads to the desired result.
\end{proof}

\begin{proof}[Proof of Lemma \ref{lemma:i-p}]
   Note that
 \bea
  \|(I- P)L\|^2
  &=& \lambda_{\max}\left(L^*(I- P)^2 L\right) = \lambda_{\max}\left(L^*(I- P) L\right) \\
  &=&\sup_{f:\|f\|_{\mcH}=1} \la L^*(I- P) L f, f \ra_{\mcH} \\
  & =& \sup_{f:\|f\|_{\mcH}=1} \la (I- P)L f,  L f \ra_{\mcH}.
 \eea
 For any $f \in \mcH,$ by using the fact that $P$ is an orthogonal projection operator onto the range of $S$ (which implies $PS=S$ and thus $S^* = S^* P$),
 \bea
\la S(S^* S + \lambda I)^{-1} S^* f,  f \ra_{\mcH}
&=& \| (S^* S  + \lambda I)^{-1/2} S^* f \|_{\mcK}^2 \\
&=& \| (S^* S  + \lambda I)^{-1/2} S^* P f \|_{\mcK}^2 \\
&\leq& \| (S^* S  + \lambda I)^{-1/2} S^*\|^2 \|P f \|_{\mcH}^2 \\
&\leq& \|P f \|_{\mcH}^2 = \la P f, f\ra_{\mcH}.
 \eea
We thus know that $ S(S^* S + \lambda I)^{-1} S^* \preccurlyeq P,$ and therefore,
$$I - P \preccurlyeq I - S(S^* S + \lambda I)^{-1} S^* = I - (SS^* + \lambda I)^{-1} SS^*= \lambda (SS^*  + \lambda I)^{-1}.$$
Consequently, we have
\bea
  \|(I- P)L\|^2
  & =& \sup_{f:\|f\|_{\mcH}=1} \la (I- P)L f, L f \ra_{\mcH} \\
  &\leq& \sup_{f:\|f\|_{\mcH}=1} \la \lambda ( SS^*  + \lambda I)^{-1} L f, L f \ra_{\mcH} \\
  &=& \lambda \| L^* (SS^*   + \lambda I)^{-1} L \| \\
  &=&  \lambda \| (SS^*   + \lambda I)^{-1/2} L L^* (SS^*   + \lambda I)^{-1/2} \| .
 \eea
\end{proof}

\begin{proof}[Proof of Lemma \ref{lem:initialerror}]
   Let $\{\sigma_i\}$ be the sequence of eigenvalues for $L.$ Since $L$ is positive, we have $0\leq \sigma_i \leq \|L\|$ for all $i$, and thus
   \bea
  \| \Pi_{k+1}^t(L) (L+ \lambda)^{\zeta}\| = \sup_{i} \prod_{l=k+1}^t (1 - \eta_l \sigma_i)(\sigma_i + \lambda)^{\zeta}.
  \eea
  For $0< \zeta \leq 1,$ we have $(\lambda+ \sigma_i)^{\zeta} \leq \lambda^{\zeta} + \sigma_i^{\zeta},$ and thus
   \bea
  \| \Pi_{k+1}^t(L) (L+ \lambda)^{\zeta}\| \leq \sup_{i} \prod_{l=k+1}^t (1 - \eta_l \sigma_i)\sigma_i + \lambda^{\zeta}.
  \eea
  We only need to bound the first term of the right-hand side of the above.
  Using the basic inequality
  \be\label{expx}
  1 + x \leq \mathrm{e}^{x} \qquad \mbox{for all } x \geq -1,
  \ee
  with $\eta_l\|L\| \leq 1$, we get
  \bea
  \sup_{i} \prod_{l=k+1}^t (1 - \eta_l \sigma_i)\sigma_i &\leq& \sup_i \exp\left\{ - \sigma_i \sum_{l=k+1}^t \eta_l\right\} \sigma_i^{\zeta} \\
  &\leq &\sup_{x \geq 0} \exp\left\{ - x  \sum_{l=k+1}^t \eta_l\right\} x^{\zeta}.
  \eea
  Using \eref{exppoly}, and from the above analysis, one can get the desired result \eref{initialerror_interm}.
\end{proof}
\begin{proof}[Proof of Lemma \ref{lem:estimate1}]
Note that
  $$ \sum_{k=1}^{t} k^{-\theta}  \leq 1 + \sum_{k=2}^t \int_{k-1}^k u^{-\theta} d u = 1  + \int_{1}^t u^{-\theta} d u = {t^{1-\theta} - \theta\over 1-\theta} ,
 $$
 which leads to the first part of the desired result.
 Similarly,
  \bea
   \sum_{k=1}^t k^{-\theta} \geq \sum_{k=1}^t \int_{k}^{k+1}u^{-\theta} d u = \int_{1}^{t+1} u^{-\theta} d u = {(t+1)^{1-\theta} - 1\over 1-\theta},
  \eea
  and by mean value theorem, $(t+1)^{1-\theta} - 1 \geq (1-\theta)t (t+1)^{-\theta} \geq (1-\theta)t^{1-\theta}/2. $
  This proves the second part of the desired result. The proof is complete.
\end{proof}
\begin{proof}[Proof of Lemma \ref{lem:estimate1a}]
  Note that
  \bea
  \sum_{k=1}^{t} k^{-\theta} = \sum_{k=1}^{t} k^{-1} k^{1-\theta} \leq t^{\max(1-\theta,0)} \sum_{k=1}^{t} k^{-1},
  \eea
  and
  \bea
  \sum_{k=1}^t k^{-1} \leq 1 + \sum_{k=2}^t \int_{k-1}^k u^{-1} du = 1 + \log t.
  \eea
\end{proof}
\begin{proof}[Proof of Lemma \ref{lem:estimate2}]
  Note that
  \bea
  \sum_{k=1}^{t-1} {1 \over t-k} k^{-q} = \sum_{k=1}^{t-1} {k^{1-q} \over (t-k)k} \leq t^{\max(1-q,0)} \sum_{k=1}^{t-1} {1 \over (t-k)k},
  \eea
 and that by Lemma \ref{lem:estimate1a},
   \bea
  \sum_{k=1}^{t-1} {1 \over (t-k)k} = {1 \over t}\sum_{k=1}^{t-1} \left({1 \over t-k} + {1 \over k}\right) = {2 \over t}\sum_{k=1}^{t-1} {1 \over k} \leq {2\over t} (1+\log t).
  \eea
\end{proof}

\section{Bounding the Empirical Risk}
This subsection is devoted to the proof of Lemma \ref{lemma:empriskB}, where the basic idea is from \cite{lin2016optimal}.
We begin with the following classical result in convex optimization.
\begin{lemma}\label{lemma:mfejer} Given any sample $\bf z,$ and $l\in \mN$, let $f \in \HKx$ be independent from $\J_l$, then
\be\label{mfejer}
\eta_l \left( \mcE_{\bf z}(f_l) - \mcE_{\bf z}(f) \right)
\leq  \|f_{l}- f\|_{\HK}^2 - \mE_{\J_l}\|f_{l+1}- f\|_{\HK}^2 + \eta_l^2 \kappa^2 \mcE_{\bf z}(f_l).
\ee
\end{lemma}
\begin{proof}
First note that both $f_{l}$ and $f$ are in $\HKx$. Thus, $f_l = \Ptx f_l$ and $f=\Ptx f.$
Subtracting both sides of \eref{eq:alg1} by $f$, and taking the square $\HK$-norm,
\bea
\|f_{t+1} - f\|_{\HK}^2 &=& \left\|\Ptx \left(f_t - f -  {\eta_t \over b} \sum_{i= b(t-1)+1}^{bt} (f_t(x_{j_i}) - y_{j_i}) K_{x_{j_i}}\right)\right\|_{\HK}^2 \\
&\leq& \left\| f_t - f -  {\eta_t \over b} \sum_{i= b(t-1)+1}^{bt} (f_t(x_{j_i}) - y_{j_i}) K_{x_{j_i}}\right\|_{\HK}^2.
\eea
Expanding the inner product of the left-hand side,
\bea
\|f_{l+1} - f\|_{\HK}^2 \leq \|f_{l} - f\|_{\HK}^2  + {\eta_l^2 \over b^2}  \left\| \sum_{i=b(l-1)+1}^{bl} ( f_l({x_{j_i}}) - y_{j_i}) K_{x_{j_i}} \right\|_{\HK}^2  \\
+ {2\eta_l \over b} \sum_{i=b(l-1)+1}^{bl} (f_l({x_{j_i}}) - y_{j_i}) \la f - f_l, K_{x_{j_i}} \ra_{\HK}.
\eea
By using the reproducing property \eref{eq:reproduce} which implies
$\la f_l,  K_{x_{j_i}} \ra_{\HK} = f_l(x_{j_i})$, we get
\bea
\|f_{l+1} - f\|_{\HK}^2 \leq \|f_{l} - f\|_{\HK}^2  + {\eta_l^2 \over b^2}  \left\| \sum_{i=b(l-1)+1}^{bl} ( f_l(x_{j_i}) - y_{j_i}) K_{x_{j_i}} \right\|_{\HK}^2  \\
+ {2\eta_l \over b} \sum_{i=b(l-1)+1}^{bl} (f_l(x_{j_i}) - y_{j_i}) (f(x_{j_i}) - f_l(x_{j_i})).
\eea
By Assumption \eref{eq:HK}, $\|K_{x_{j_i}}\|_{\HK} \leq \kappa$, and thus
\bea
\left\| \sum_{i=b(l-1)+1}^{bl} (\la f_l, K_{x_{j_i}} \ra_{\HK} - y_{j_i}) K_{x_{j_i}} \right\|_{\HK}^2
&\leq& \left(\sum_{i=b(l-1)+1}^{bl} |f_l(x_{j_i}) - y_{j_i}| \kappa\right)^2 \\
&\leq& \kappa^2 b \sum_{i=b(l-1)+1}^{bl} (f_l(x_{j_i}) - y_{j_i})^2,
\eea
where for the last inequality, we used Cauchy-Schwarz inequality.
Thus,
\bea
\|f_{l+1}- f\|_{\HK}^2
 \leq \|f_{l} - f \|_{\HK}^2  +  { \eta_l^2 \kappa^2 \over b} \sum_{i=b(l-1)+1}^{bl} (f_l(x_{j_i}) - y_{j_i})^2  \\
 + {2\eta_l \over b} \sum_{i=b(l-1)+1}^{bl} (f_l(x_{j_i}) - y_{j_i}) (f(x_{j_i}) - f_l(x_{j_i}) ).
\eea
Using the basic inequality $a(b-a) \leq (b^2 - a^2)/2,\forall a,b \in \mR,$
\bea
\|f_{l+1}- f\|_{\HK}^2
\leq  \|f_{l}- f\|_{\HK}^2  + { \eta_l^2\kappa^2 \over b} \sum_{i=b(l-1)+1}^{bl} (f_l(x_{j_i}) - y_{j_i})^2  \\
  + {\eta_l \over b} \sum_{i=b(l-1)+1}^{bl} \left( (f(x_{j_i})  - y_{j_i})^2 - ( f_l(x_{j_i}) - y_{j_i})^2 \right).
\eea
Noting that $f_{l}$ and $f$ are independent from $\J_l$, and taking the expectation on both sides with respect to $\J_l,$
\bea
\mE_{\J_l}\|f_{l+1}- f\|_{\HK}^2
\leq  \|f_{l}- f\|_{\HK}^2  + \eta_l^2 \kappa^2 \mcE_{\bf z}(f_l)  + \eta_l \left( \mcE_{\bf z}(f) - \mcE_{\bf z}(f_l) \right),
\eea
which leads to the desired result by rearranging terms. The proof is complete.
\end{proof}
Using the above lemma and a decomposition related to the weighted averages and the last iterates from \cite{shamir2013stochastic,lin2015iterative}, we can prove the following relationship.
\begin{lemma}\label{lemma:empiricalRelat}
  Let $\eta_1 \kappa^2 \leq 1/2$ for all $t \in \mN.$ Then
  \be\label{empiricalRelat}
   \eta_t \mE_{\J} [\mcE_{\bf z}(f_t)]  \leq 4\mcE_{\bf z}(0) {1 \over t} \sum_{l=1}^t \eta_l + 2\kappa^2  \sum_{k=1}^{t-1}{1 \over k(k+1)}  \sum_{i=t-k}^{t-1} \eta_i^2\mE_{\J} [\mcE_{\bf z}(f_i)].
\ee
\end{lemma}
\begin{proof}
 For $k=1, \cdots, t-1$,
  \bea
  && {1 \over k} \sum_{i=t-k+1}^{t}  \eta_i \mE_{\J}[\mcE_{\bf z}(f_i)] - {1 \over k+1} \sum_{i=t-k}^t \eta_i \mE_{\J}[\mcE_{\bf z}(f_i)] \\
   &=& {1 \over k(k+1)} \left\{ (k+1)\sum_{i=t-k+1}^{t} \eta_i \mE_{\J}[\mcE_{\bf z}(f_i)] - k \sum_{i=t-k}^t \eta_i \mE_{\J}[\mcE_{\bf z}(f_i)] \right\}\\
 & =& {1 \over k(k+1)} \sum_{i=t-k+1}^{t} (\eta_i \mE_{\J}[\mcE_{\bf z}(f_i)] -  \eta_{t-k} \mE_{\J}[\mcE_{\bf z}(f_{t-k})]) .
  \eea
  Summing over $k=1, \cdots, t-1$, and rearranging terms, we get \cite{lin2015iterative}
\bea
   \eta_t \mE_{\J}[\mcE_{\bf z}(f_t)]  = {1 \over t} \sum_{i=1}^t \eta_i \mE_{\J} [\mcE_{\bf z}(f_i)] + \sum_{k=1}^{t-1} {1 \over k(k+1)} \sum_{i=t-k+1}^{t} (\eta_i \mE_{\J} [\mcE_{\bf z}(f_i)] -  \eta_{t-k} \mE_{\J}[\mcE_{\bf z}(f_{t-k})]) .
\eea
Since $\{\eta_t\}_t$ is decreasing and $\mE_{\J}[\mcE_{\bf z}(f_{t-k})]$ is non-negative, the above can be relaxed as
\be\label{decomposition}
   \eta_t \mE_{\J} [\mcE_{\bf z}(f_t)]  \leq {1 \over t} \sum_{i=1}^t \eta_i \mE_{\J}[\mcE_{\bf z}(f_i)] + \sum_{k=1}^{t-1} {1 \over k(k+1)} \sum_{i=t-k+1}^{t} \eta_i  \mE_{\J}[\mcE_{\bf z}(f_i) -  \mcE_{\bf z}(f_{t-k})] .
\ee
In the rest of the proof, we will upper bound the last two terms of the above.

To bound the first term of the right side of \eref{decomposition}, we apply Lemma \ref{lemma:mfejer} with $f=0$ to get
\bea
\eta_l \mE_{\J} \left( \mcE_{\bf z}(f_l) - \mcE_{\bf z}(0) \right)
\leq  \mE_{\J} [\|f_{l}\|_{\HK}^2 - \|f_{l+1}\|_{\HK}^2] + \eta_l^2 \kappa^2 \mE_{\J}[\mcE_{\bf z}(f_l)].
\eea
Rearranging terms,
\bea
\eta_l (1 - \eta_l \kappa^2) \mE_{\J}[ \mcE_{\bf z}(f_l)]
\leq  \mE_{\J} [\|f_{l}\|_{\HK}^2 - \|f_{l+1}\|_{\HK}^2] + \eta_l \mcE_{\bf z}(0).
\eea
 It thus follows from the above and  $\eta_l \kappa^2 \leq 1/2$ that
\bea
\eta_l \mE_{\J}[\mcE_{\bf z}(f_l)]/2  \leq \mE_{\J}[\|f_{l}\|_{\HK}^2  - \|f_{l+1}\|_{\HK}^2]    + \eta_l \mcE_{\bf z}(0).
\eea
Summing up over $l=1,\cdots,t,$
\bea
\sum_{l=1}^t \eta_l \mE_{\J}[\mcE_{\bf z}(f_l)]/2  \leq  \mE_{\J}[\|f_1\|_{\HK}^2  - \|f_{t+1}\|_{\HK}^2]    + \mcE_{\bf z}(0) \sum_{l=1}^t \eta_l .
\eea
Introducing with $f_1=0, \|f_{t+1}\|_{\HK}^2\geq 0$, and then multiplying both sides by $2/t,$ we get
\be\label{averageBound}
{1 \over t} \sum_{l=1}^t \eta_l \mE_{\J}[\mcE_{\bf z}(f_l)]  \leq  2 \mcE_{\bf z}(0) {1 \over t} \sum_{l=1}^t \eta_l .
\ee

It remains to bound the last term of \eref{decomposition}. Let $k \in [t-1]$ and $i\in \{t-k,\cdots, t\}.$ Note that given the sample $\bf z,$ $f_i$ is depending only on $\J_1,\cdots, \J_{i-1}$ when $i>1$ and $f_1=0.$
Thus, we can apply Lemma \ref{lemma:mfejer} with $f=f_{t-k}$ to derive
\bea
\eta_i \left( \mcE_{\bf z}(f_i) - \mcE_{\bf z}(f_{t-k}) \right)
\leq   \|f_{i}-f_{t-k}\|_{\HK}^2 - \mE_{\J_i}\|f_{i+1}-f_{t-k}\|_{\HK}^2 + \eta_i^2 \kappa^2 \mcE_{\bf z}(f_i).
\eea
Therefore,
\bea
\eta_i \mE_{\J}\left[ \mcE_{\bf z}(f_i) - \mcE_{\bf z}(f_{t-k}) \right]
\leq   \mE_{\J} [\|f_{i}-f_{t-k}\|_{\HK}^2 -\|f_{i+1}-f_{t-k}\|_{\HK}^2] + \eta_i^2 \kappa^2 \mE_{\J} [\mcE_{\bf z}(f_i)].
\eea
Summing up over $i=t-k,\cdots, t,$
\bea
\sum_{i=t-k}^t \eta_i \mE_{\J} \left [ \mcE_{\bf z}(f_i) - \mcE_{\bf z}(f_{t-k}) \right]
\leq   \kappa^2 \sum_{i=t-k}^t \eta_i^2 \mE_{\J} [\mcE_{\bf z}(f_i)].
\eea
Note that the left hand side is exactly $\sum_{i=t-k+1}^t \eta_i \mE_{\J}  \left[ \mcE_{\bf z}(f_i) - \mcE_{\bf z}(f_{t-k}) \right]$.
We thus know that the last term of \eref{decomposition} can be upper bounded by
\bea
&&\kappa^2  \sum_{k=1}^{t-1}{1 \over k(k+1)}  \sum_{i=t-k}^t \eta_i^2\mE_{\J} [\mcE_{\bf z}(f_i)] \\
&=& \kappa^2  \sum_{k=1}^{t-1}{1 \over k(k+1)}  \sum_{i=t-k}^{t-1} \eta_i^2\mE_{\J} [\mcE_{\bf z}(f_i)] + \kappa^2 \eta_t^2\mE_{\J} [\mcE_{\bf z}(f_t)] \sum_{k=1}^{t-1}{1 \over k(k+1)}.
\eea
Using the fact that
\bea
\sum_{k=1}^{t-1}{1 \over k(k+1)} = \sum_{k=1}^{t-1} \left({1 \over k} - {1 \over k+1}\right) = 1 - {1\over t} \leq 1,
\eea
and $\kappa^2 \eta_t \leq 1/2,$
we get that the last term of \eref{decomposition} can be bounded as
\bea
&&\sum_{k=1}^{t-1} {1 \over k(k+1)} \sum_{i=t-k+1}^{t} \eta_i ( \mE_{\J}[\mcE_{\bf z}(f_i)] -  \mE_{\J}[\mcE_{\bf z}(f_{t-k})]) \\
&\leq &\kappa^2  \sum_{k=1}^{t-1}{1 \over k(k+1)}  \sum_{i=t-k}^{t-1} \eta_i^2\mE_{\J} [\mcE_{\bf z}(f_i)] +  \eta_t\mE_{\J} [\mcE_{\bf z}(f_t)]/2.
\eea
Plugging the above and \eref{averageBound} into the decomposition \eref{decomposition}, and rearranging terms
\bea
   \eta_t \mE_{\J}[\mcE_{\bf z}(f_t)]/2  \leq 2\mcE_{\bf z}(0){1 \over t} \sum_{l=1}^t \eta_l + \kappa^2  \sum_{k=1}^{t-1}{1 \over k(k+1)}  \sum_{i=t-k}^{t-1} \eta_i^2\mE_{\J} [\mcE_{\bf z}(f_i)],
\eea
which leads to the desired result by multiplying both sides by $2$. The proof is complete.
\end{proof}

We also need to the following lemma, whose proof can be done by using an induction argument.
\begin{lemma}
  \label{lemma:induction}
  Let $\{u_t\}_{t=1}^T$, $\{A_t\}_{t=1}^T$ and  $\{B_t\}_{t=1}^T$ be three sequences of non-negative numbers such that
  $u_1 \leq A_1$ and
  \be\label{inductionAssu}
  u_t \leq A_t  + B_t \sup_{i \in [t-1]} u_i,\qquad \forall t\in\{2,3,\cdots, T\}.
  \ee
  Let $\sup_{t \in [T]} B_t \leq B < 1.$
  Then for all $t \in [T],$
  \be\label{inductionConse}
  \sup_{k \in [t]} u_t \leq   {1 \over 1 - B} \sup_{k \in [t]} A_k.
  \ee
\end{lemma}
\begin{proof}
  When $t=1,$ \eref{inductionConse} holds trivially since $u_1 \leq A_1$ and $B< 1$. Now assume for some $t \in \mN$ with $ 2\leq  t \leq T,$
  \bea
  \sup_{i \in [t-1]} u_i \leq {1 \over 1 - B} \sup_{i \in [t-1]} A_i.
  \eea
  Then, by \eref{inductionAssu}, the above hypothesis,  and $B_t \leq B$, we have
  \bea
  u_t \leq A_t  + B_t \sup_{i \in [t-1]} u_i \leq A_t + {B_t \over 1 - B} \sup_{i \in [t-1]} A_i \leq \sup_{i \in [t]} A_i \left(1 + {B_t \over 1 - B} \right) \leq \sup_{i \in [t]} A_i {1 \over 1 - B}.
  \eea
  Consequently,
  \bea
  \sup_{k \in [t]} u_t \leq   {1 \over 1 - B} \sup_{k \in [t]} A_k,
  \eea
  thereby showing that indeed \eref{inductionConse} holds for $t$.
  By mathematical induction, \eref{inductionConse} holds for every $t\in [T].$
  The proof is complete.
\end{proof}
Now we can bound $\mE_{\J}[\mcE_{\bf z}(f_k)]$ as follows.

\begin{proof}[Proof of Lemma \ref{lemma:empriskB}]
By Lemma \ref{lemma:empiricalRelat}, we have \eref{empiricalRelat}. Dividing both sides by $\eta_t$, we can relax the inequality as
\bea
    \mE_{\J}[\mcE_{\bf z}(f_t)]  \leq 4\mcE_{\bf z}(0) {1 \over \eta_t t} \sum_{l=1}^t \eta_l + 2\kappa^2 {1 \over \eta_t}  \sum_{k=1}^{t-1}{1 \over k(k+1)}  \sum_{i=t-k}^{t-1} \eta_i^2 \sup_{i \in [t-1]}\mE_{\J} [\mcE_{\bf z}(f_i)].
\eea
In Lemma \ref{lemma:induction},
we let  $u_t = \mE_{\J}[\mcE_{\bf z}(f_t)]$, $A_t = 4\mcE_{\bf z}(0){1 \over \eta_t t} \sum_{l=1}^t \eta_l$ and $$B_t = 2\kappa^2 {1 \over \eta_t}  \sum_{k=1}^{t-1}{1 \over k(k+1)}  \sum_{i=t-k}^{t-1} \eta_i^2.$$
Condition \eref{empriskBCon} guarantees that $\sup_{t \in [T]} B_t \leq 1/2.$ Also, $u_1 \leq A_1$ as $f_1=0.$
 Thus, \eref{inductionConse} holds, and the desired result follows by plugging with $B=1/2.$
The proof is complete.
\end{proof}
\end{document}